\documentclass{article}

\usepackage{silence}
\PassOptionsToPackage{numbers}{natbib}

\usepackage[final]{neurips_2020}

\usepackage[utf8]{inputenc} 
\usepackage[T1]{fontenc}    
\usepackage{hyperref}       
\usepackage{url}            
\usepackage{booktabs}       
\usepackage{amsfonts}       
\usepackage{nicefrac}       
\usepackage{microtype}      

\usepackage{amsmath,amsfonts,amsthm,amssymb}

\usepackage{color}
\usepackage{comment}

\usepackage{thm-restate}
\usepackage{enumitem}
\usepackage{nicefrac}
\usepackage{caption}
\usepackage{booktabs,siunitx,caption}
\usepackage{wrapfig}
\usepackage{graphicx}
\usepackage{subfigure}

\newcommand{\A}{\mathcal{A}}

\newcommand{\R}{\mathbb{R}}

\renewcommand{\P}{\mathcal{P}}

\numberwithin{equation}{section}
\numberwithin{figure}{section}
\theoremstyle{plain}
	\newtheorem{theorem}{Theorem}[section]

	\newtheorem{lemma}[theorem]{Lemma}

	\newtheorem{fact}[theorem]{Fact}

\theoremstyle{definition}

	\newtheorem*{remark*}{Remark}

\let \cite \citep

\title{A Study on Encodings for Neural Architecture Search}

\author{Colin White\\
        Abacus.AI\\
        San Francisco, CA 94103\\
        \texttt{colin@abacus.ai}\\
       \And
       Willie Neiswanger\\
       Stanford University and Petuum, Inc.\\
       Stanford, CA 94305\\
       \texttt{neiswanger@cs.stanford.edu}\\
       \And
        Sam Nolen\\
        Abacus.AI\\
        San Francisco, CA 94103\\
        \texttt{sam@abacus.ai}\\
        \And
        Yash Savani\\
        Abacus.AI\\
        San Francisco, CA 94103\\
        \texttt{yash@abacus.ai}\\
       }

\begin{document}

\maketitle

\begin{abstract}
Neural architecture search (NAS) has been extensively studied in the 
past few years. 
A popular approach is to represent each neural architecture in the search
space as a directed acyclic graph (DAG), and then search over all DAGs
by encoding the adjacency matrix and list of operations
as a set of hyperparameters. 
Recent work has demonstrated that even small changes to the 
way each architecture is encoded can have a significant effect 
on the performance of NAS algorithms \citep{bananas, nasbench}.

In this work, we present the first formal study on the effect of architecture
encodings for NAS, including a theoretical grounding and an empirical study.
First we formally define architecture encodings and give a theoretical 
characterization on the scalability of the encodings we study.
Then we identify the main encoding-dependent subroutines which NAS algorithms 
employ, running experiments to show which encodings work best with each
subroutine for many popular algorithms. The experiments act as an
ablation study for prior work, disentangling the algorithmic and encoding-based
contributions, as well as a guideline for future work.
Our results demonstrate that NAS encodings are an important design decision which
can have a significant impact on overall performance.\footnote{
Our code is available at \url{https://github.com/naszilla/naszilla}.}
\end{abstract}

\section{Introduction} \label{sec:intro}

In the past few years, the field of neural architecture search (NAS)
has seen a steep rise in interest~\citep{nas-survey},
due to the promise of automatically designing specialized neural architectures
for any given problem.
Techniques for NAS span 
evolutionary search, Bayesian optimization, reinforcement learning,
gradient-based methods, and neural predictor methods.
Many NAS instantiations can be described by the optimization problem $\min_{a\in A}f(a)$,
where $A$ denotes a large set of neural architectures, and $f(a)$ denotes
the objective function of interest for $a$, which is usually a combination of validation
accuracy, latency, or number of parameters.
A popular approach is to describe each neural architecture $a$ as a labeled 
directed acyclic graph (DAG), where each node or edge represents an operation.

Due to the complexity of DAG structures
and the large size of the space, neural architecture search is typically 
a highly non-convex, challenging optimization problem.
A natural consideration when designing a NAS algorithm is therefore,
\emph{how should we encode the neural architectures to maximize performance?}
For example, NAS algorithms may involve manipulating or perturbing architectures,
or training a model to predict the accuracy of a given architecture; 
as a consequence, the representation of the DAG-based architectures may significantly change 
the outcome of these subroutines.
The majority of prior work has not explicitly considered this question,
opting to use a standard encoding consisting of the
adjacency matrix of the DAG along with a list of the operations.
Two recent papers have shown that even small changes to the architecture encoding can make a 
substantial difference in the final performance of the NAS 
algorithm~\cite{bananas, nasbench}.
It is not obvious how to formally define an encoding for NAS, as prior work defines
encodings in different ways, inadvertently using encodings which are incompatible
with other NAS algorithms.

In this work, we provide the first formal study on NAS encoding schemes,
including a theoretical grounding as well as a set of
experimental results.
We define an encoding as a multi-function from an architecture 
to a real-valued tensor.
We define a number of common encodings from prior
work, identifying adjacency matrix-based encodings~\citep{zoph2017neural, nasbench, wen2019neural} 
and path-based encodings~\citep{bananas, npenas, talbi2020optimization}
as two main paradigms.
Adjacency matrix approaches represent the architecture as a list of edges
and operations, while path-based approaches represent the architecture as
a set of paths from the input to the output.
We theoretically characterize the scalability of each encoding by
quantifying the information loss from truncation.
This characterization is particularly interesting for path-based encodings,
which we find to exhibit a phase change at $r^{k/n}$, where $r$ is the number of possible
operations, $n$ is the number of nodes, and $k$ is the expected number of edges.
In particular, we show that when the size of the path encoding is greater than $r^{2k/n}$,
barely any information is lost, but below $r^{k/(2n)}$, nearly all information is lost.
We empirically verify these findings.

Next, we identify three major encoding-dependent subroutines 
used in NAS algorithms:
\emph{sample random architecture}, \emph{perturb architecture}, and
\emph{train predictor model}.
We show which of the encodings perform best for each subroutine by testing
each encoding within each subroutine for many popular NAS algorithms.
Our experiments retroactively provide an ablation study for prior
work by disentangling the algorithmic contributions from the
encoding-based contributions.
We also test the ability of a neural predictor to generalize to new search spaces,
using a given encoding.
Finally, for encodings in which multiple architectures can map to the
same encoding, we evaluate the average standard deviation of accuracies for the equivalence class of architectures defined by each encoding.

Overall, our results show that NAS encodings are an important design decision
which must be taken into account not only at the algorithmic level, but at the
subroutine level, and which can have a significant impact on the final performance.
Based on our results, we lay out recommendations for which encodings to use within
each NAS subroutine.
Our experimental results follow the guidelines in the
recently released NAS research checklist~\cite{lindauer2019best}.
In  particular, we experiment on two popular NAS benchmark datasets,
and we release our code.

\paragraph{Our contributions.} We summarize our main contributions below.
\begin{itemize} [topsep=0pt, itemsep=2pt, parsep=0pt, leftmargin=5mm]
    \item We demonstrate that
    the choice of encoding is an important, nontrivial question
    that should be considered not only at the algorithmic level, 
    but at the subroutine level.
    \item We give a theoretical grounding for NAS encodings, including a 
    characterization of the scalability of each encoding.
    \item We give an experimental study of architecture encodings for 
    NAS algorithms, disentangling the algorithmic contributions from the 
    encoding-based contributions of prior work, and laying out recommendations
    for best encodings to use in different settings as guidance for future work.
\end{itemize}

\section{Related Work} \label{sec:related}

\paragraph{Neural architecture search.}
NAS has been studied for at least two decades and has received significant
attention in recent years
\citep{kitano1990designing, stanley2002evolving, zoph2017neural}. 
Some of the most popular techniques for NAS include evolutionary algorithms \citep{maziarz2018evolutionary}, 
reinforcement learning \citep{enas, efficientnets}, 
Bayesian optimization~\citep{nasbot}, 
gradient descent~\citep{darts},
neural predictors~\citep{wen2019neural},
and local search~\citep{white2020local}.
Recent papers have highlighted the need for fair and reproducible NAS comparisons~\citep{randomnas, nasbench, lindauer2019best}.
See the recent survey~\citep{nas-survey} for more information on NAS research.

\paragraph{Encoding schemes.}
Most prior NAS work has used the adjacency matrix encoding,~\citep{zoph2017neural, nasbench, darts},
which consists of the adjacency matrix together with a list of the 
operations on each node.
A continuous-valued variant has been shown to be more effective for some
NAS algorithms~\citep{nasbench}.
The path encoding is a popular choice for neural predictor 
methods~\citep{bananas, npenas, talbi2020optimization},
and it was shown that truncating the path encoding leads to a small information loss~\citep{bananas}.

Some prior work uses graph convolutional networks (GCN)
as a subroutine in NAS~\citep{shi2019multi, dvae}, 
which requires retraining for each new dataset or search space.
Other work has used intermediate encodings to reduce the complexity of the 
DAG~\citep{stanleyhypercube, irwin2019graph}, or added summary statistics
to the encoding of feedforward networks~\citep{sun2019evolving}.
To the best of our knowledge, no paper has conducted
a formal study of encodings involving more than two encodings.

\begin{figure*}
\centering %
\includegraphics[width=\textwidth]{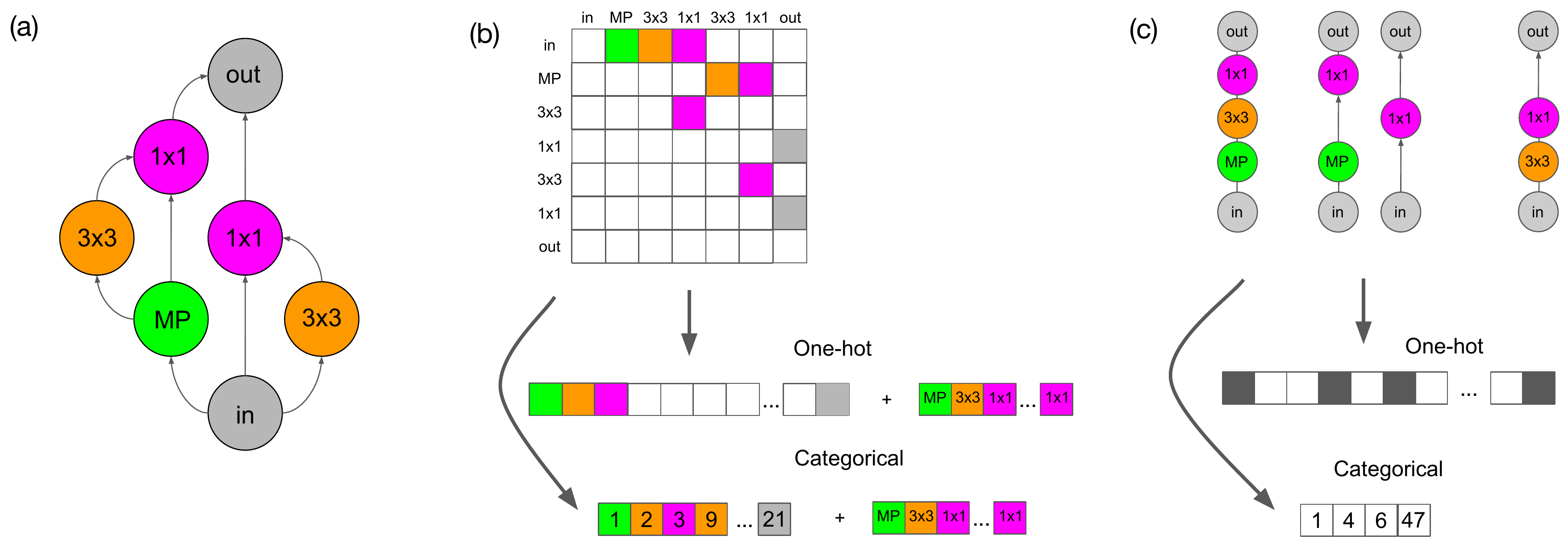}
\caption{
(a) An example neural architecture $a$. (b) An adjacency matrix representation of $a$, showing two encodings. (c) A path-based representation of $a$, showing two encodings.
}
\label{fig:encodings_viz}
\end{figure*}

\section{Encodings for NAS} \label{sec:prelim}

We denote a set of neural architectures $a$ by $A$ (called a search space),
and we define an objective function $\ell:A\rightarrow\R$, where $\ell(a)$
is typically a combination of the accuracy and the model complexity.
We define a neural architecture encoding as an integer $d$ and a
multifunction $e: A \rightarrow \mathbb{R}^d$
from a set of neural architectures $A$ to a $d$-dimensional Euclidean space
$\mathbb{R}^d$, and we define a NAS algorithm $\A$ as a procedure which takes as input a triple
$(A, \ell,e)$, and outputs an architecture $a$, with the goal that $\ell(a)$ is
as close to $\max_{a\in A} \ell(a)$ as possible.
Based on this definition, we consider an encoding $e$ to be a fixed transformation, independent of $\ell$. 
In particular, NAS components that use $\ell$ to learn a transformation of an input architecture such as graph convolutional networks (GCN) or variational autoencoders (VAE), 
are considered part of the NAS algorithm rather than the encoding.
This is consistent with prior definitions of encodings~\citep{talbi2020optimization, nasbench}.
However, we do still experimentally compare the fixed encodings with GCNs and VAEs 
in Section~\ref{sec:experiments}.

We define eight encodings split into two popular paradigms: adjacency matrix-based
and path-based encodings.
We assume that each architecture is represented by a DAG with 
at most $n$ nodes, at most $k$ edges, 
at most $P$ paths from input to output,
and $q$ choices of operations on each node.
We focus on the case where nodes represent operations, though our analysis extends similarly to formulations where edges represent operations.
Most of the following encodings have been defined in prior
work~\citep{nasbench, bananas, talbi2020optimization}, 
and we will see in the next section that each encoding
is useful for some part of the NAS pipeline.

\paragraph{Adjacency matrix encodings.}
We first consider a class of encodings that are based on representations
of the adjacency matrix. These are the most common types of encodings used in 
current NAS research.

\begin{itemize}
    \item The \emph{one-hot adjacency matrix encoding} is created by row-major 
    vectorizing (i.e.\ flattening) the architecture adjacency matrix and concatenating it with a list of node operation labels. Each position in the operation list is a 
    single integer-valued feature, where each operation is denoted by a different 
    integer. The total dimension is $n(n-1)/2 + n$. See Figure~\ref{fig:encodings_viz}.
    \item 
    In the \emph{categorical adjacency matrix encoding}, the adjacency matrix is
    first flattened (similar to the one-hot encoding described previously),
    and is then defined as a list of the indices each of which specifies one of the
    $n(n-1)/2$ possible edges in the adjacency  matrix.
    To ensure a fixed length encoding, each architecture is represented by $k$
    features, where $k$ is the maximum number of possible edges.
    We again concatenate this representation with a list of operations, yielding
    a total dimensionality of $k+n$. See Figure~\ref{fig:encodings_viz}.
    \item 
    Finally, the \emph{continuous adjacency matrix encoding} is similar to the
    one-hot encoding, but each of the features for each edge can take on any real
    value in $[0,1]$, rather than just $\{0,1\}.$
    We also add a feature representing the number of edges, $1\leq K\leq k$.
    The list of operations is encoded the same way as before.
    The architecture is created by choosing the $K$ edges with the largest
    continuous features. The dimension is $n(n-1)/2+n+1.$
\end{itemize}
The disadvantage of adjacency matrix-based encodings is that nodes are arbitrarily
assigned indices in the matrix, which means one architecture can have many
different representations (in other words, $e^{-1}$ is not one-to-one).
See Figure~\ref{fig:non_onto} (b).

\paragraph{Path-based encodings.}
Path-based encodings are representations of a neural architecture that are based on the set of paths from input to output that are present within the architecture DAG.
\begin{itemize}
    \item 
    The \emph{one-hot path encoding} is created by giving a binary feature to each
    possible path from the input node to the output node in the DAG (for example:
    \texttt{input}--\texttt{conv1x1}--\texttt{maxpool3x3}--\texttt{output}). 
    See Figure~\ref{fig:encodings_viz}.
    The total dimension is $\sum_{i=0}^n q^i=(q^{n+1}-1)/(q-1)$.
    The \emph{truncated one-hot path encoding}, simply truncates this encoding
    to only include paths of length $x$. The new dimension is $\sum_{i=0}^x q^i$.
    \item
    The \emph{categorical path encoding}, is defined as a list of indices each
    of which specifies one of the  $\sum_{i=0}^n q^i$ possible paths.
    See Figure~\ref{fig:encodings_viz}.
    \item
    The \emph{continuous path encoding} consists of a real-valued feature $[0,1]$ for
    each potential path, as well as a feature representing the number of paths.
    Just like the one-hot path encoding, the continuous path encoding can
    be truncated.
\end{itemize}

Path-based encodings have the advantage that nodes are not arbitrarily 
assigned indices, and also that isomorphisms are automatically mapped to the 
same encoding. Path-based encodings have the disadvantage that different
architectures can map to the same encoding ($e$ is not one-to-one).
See Figure~\ref{fig:non_onto} (c).

\begin{figure*}
\centering %
\includegraphics[width=\textwidth]{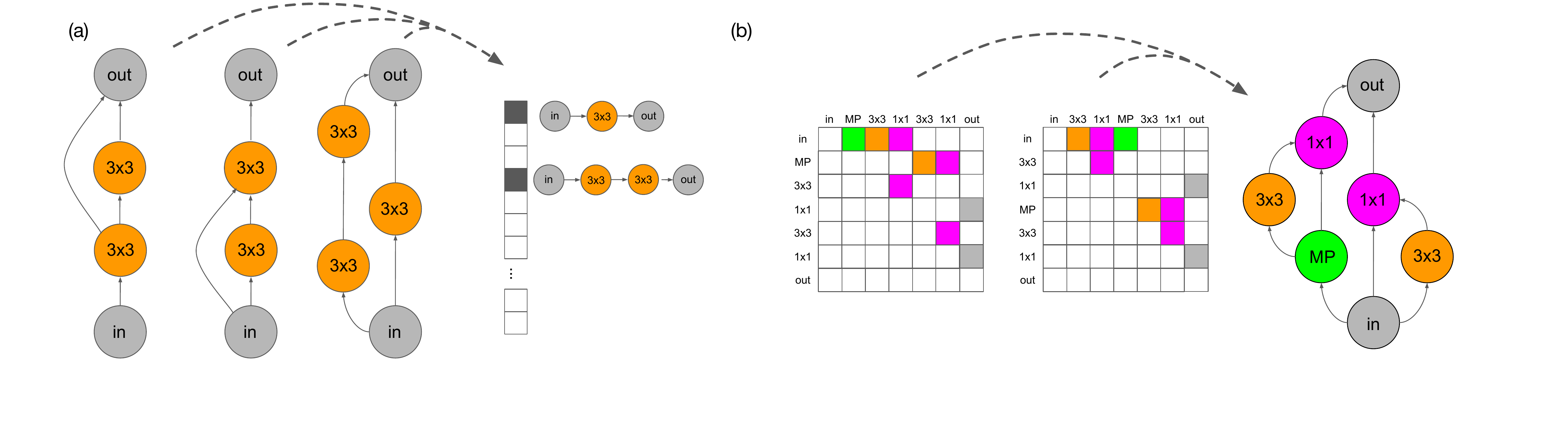}
\caption{
(a) An example of three architectures that map to the same path encoding.
(b) An example of two adjacency matrices that map to the same architecture.
}
\label{fig:non_onto}
\end{figure*}


\subsection{The scalability of encodings}
In this section, we discuss the scalability of the NAS encodings with
respect to architecture size.
We focus on the one-hot variants of the encodings,
but our analysis extends to all encodings.
We show that the path encoding can be truncated significantly
while maintaining its performance, while the adjacency matrix cannot be truncated at all
without sacrificing performance,
and we back up our theoretical results with experimental observations 
in the next section.
In prior work, the one-hot path encoding has been shown to be effective on smaller
benchmark NAS datasets~\citep{npenas, bananas}, but it has been
questioned whether its exponential $\Theta(q^n)$ length
allows it to perform well on very large search spaces~\citep{talbi2020optimization}.
However, a counter-arguement is as follows.
The vast majority of features correspond to single line
paths using the full set of nodes. This type of architecture is not common during
NAS algorithms, nor is it likely to be effective in real applications.
Prior work has made the first steps in showing that truncating the path
encoding does not harm the performance of NAS algorithms~\citep{bananas}.

Consider the popular \emph{sample random architecture} method:
given $n,~r,$ and $k\leq\frac{n(n-1)}{2},$
\emph{(1)} choose one of $r$ operations for each node from $1$ to $n$;
\emph{(2)} for all $i<j$, add an edge from node $i$ to node $j$
with probability $\frac{2k}{n(n-1)}$;
\emph{(3)} if there is no path from node 1 to node $n$, 
\texttt{goto}\texttt{(1)}.
Given a random graph $G_{n,k,r}$ outputted by this method,
let $a_{n,k,\ell}$ denote the expected number of paths
from node 1 to node $n$ of length $\ell$ in $G_{n,k,r}$.
We define
\begin{equation*}
    b(k, x)=\frac{\sum_{\ell=1}^x a_{n,k,\ell}}{\sum_{\ell=1}^n a_{n,k,\ell}}.
\end{equation*}
Given $n<k<n(n-1)/2$ and $0<x<n$, $b(k,x)$ represents the expected fraction of 
paths of length at most $x$ in $G_{n,k,r}$.
Say that we truncate the path encoding to only include paths of length at most $x$.
If $b(k,x)$ is very close to one, then the truncation will result in very little
information loss because nearly all paths in a randomly drawn architecture are
length at most $x$ with high probability.
However, if $b(k,x)$ is bounded away from 1 by some constant, there may
not be enough information in the truncated path encoding to effectively run a NAS
algorithm.

Prior work has shown that $b(k,x)>1-1/n^2$ when $k<n+O(1)$ and 
$x>\log n$~\citep{bananas}.
However, no bounds for $b(k,x)$ are known when $k$ is larger than a constant
added to $n$.
Now we present our main result for the path encoding, 
which gives a full characterization of
$b(k,x)$ up to constant factors. Interestingly, we show that $b(k,x)$
exhibits a phase transition at $x=k/n.$
What this means is, for the purposes of NAS, truncating the path encoding
to length $r^{k/n}$ contains almost exactly the same information as the full path
encoding, and it cannot be truncated any smaller.
In particular, if $k\leq n\log n,$
the truncated path encoding can be length $n$, 
which is smaller than the one-hot adjacency matrix encoding.
We give the details of the proofs from this section in
Appendix~\ref{app:prelim}.

\begin{restatable}{rethm}{characterization}\label{thm:characterization}
Given $10\leq n\leq k \leq \frac{n(n-1)}{2}$, and $c>3$,
for $x> 2ec\cdot\frac{k}{n}$, $b(k,x)>1-c^{-x+1}$,
and for $x<\frac{1}{2ec}\cdot\frac{k}{n},~b(k,x)<-2^{\frac{k}{2n}}.$
\end{restatable}

\begin{proof}[\textbf{Proof sketch.}]
Let $G'_{n,k,r}$ denote a random graph after step \emph{(2)} of 
\emph{sample random architecture}. Then $G'_{n,k,r}$ may not contain a
path from node 1 to node $n$.
Let $a'_{n,k,\ell}$ denote the expected number of paths of length $\ell$ in
$G'_{n,k,r}$.
Say that a graph is \emph{valid} if it contains a path from node 1 to node $n$.
Then
\begin{equation*}
 a'_{n,k,\ell}=0\cdot (1-P(G'_{n,k,r}\text{ is valid}))+a_{n,k,\ell}\cdot P(G'_{n,k,r}\text{ is valid}),   
\end{equation*}
so $a_{n,k,\ell}=a'_{n,k,\ell}/P(G'_{n,k,r}\text{ is valid}).$
Then
\begin{equation*}
    b(k, x)=\frac{\sum_{\ell=1}^x a_{n,k,\ell}}{\sum_{\ell=1}^n a_{n,k,\ell}}
    =\frac{\sum_{\ell=1}^x a'_{n,k,\ell}/P(G'_{n,k,r}\text{ is valid})}{\sum_{\ell=1}^n a'_{n,k,\ell}/P(G'_{n,k,r}\text{ is valid})}
    =\frac{\sum_{\ell=1}^x a'_{n,k,\ell}}{\sum_{\ell=1}^n a'_{n,k,\ell}}.
\end{equation*}

\begin{equation*}
\text{Now we claim }~
\frac{2k}{n(n-1)}\left(\frac{2k(n-2)}{(\ell-1)n(n-1)}\right)^{\ell-1}
\leq a_{n,k,\ell}\leq
\frac{2k}{n(n-1)}\left(\frac{2ek(n-2)}{(\ell-1)n(n-1)}\right)^{\ell-1}.
\end{equation*}
This is because on a path from node 1 to $n$ of length $\ell$,
there are $\binom{n-2}{\ell-1}$ choices of intermediate
nodes from 1 to $n$. Once the nodes are chosen, we need all $\ell$
edges between the nodes to exist, and each edge exists independently 
with probability $\frac{2}{n(n-1)}\cdot k.$
Then we use the well-known binomial inequalities
$\left(\frac{n}{\ell}\right)^\ell \leq \binom{n}{\ell}
\leq \left(\frac{en}{\ell}\right)^\ell$
to finish the claim.

To prove the first part of Theorem~\ref{thm:characterization},
given $x> 2ec\cdot\frac{k}{n},$
we must upper bound $\sum_{\ell={x+1}}^n a'_{n,k,\ell}$
and lower bound $\sum_{\ell=1}^x a'_{n,k,\ell}$.
To lower bound $\sum_{\ell=1}^x a'_{n,k,\ell}$, 
we use $x> 2ec\cdot\frac{k}{n}$ with the claim:
\begin{align*}
\sum_{\ell=x+1}^n a_{n,k,\ell}    
\leq \sum_{\ell=x+1}^n \frac{2k}{n(n-1)}
\left(\frac{2ek(n-2)}{(\ell-1)n(n-1)}\right)^{\ell-1}
&\leq \frac{2k}{n(n-1)}\sum_{\ell=x+1}^n \left(\frac{1}{c}\right)^{\ell-1}\\
&\leq\left(\frac{2k}{n(n-1)}\right)
\left(\frac{1}{c}\right)^{x-1}
\end{align*}

We also have $a_{n,k,1}=\frac{2k}{n(n-1)}$ because there is just one path
of length 1: the edge from the input node to the output node.
Therefore, we have
\begin{equation*}
b(k,x)=\frac{\sum_{\ell=1}^x a_{n,k,\ell}}{\sum_{\ell=1}^n a_{n,k,\ell}}\\
\geq\frac{a_{n,k,1}}{a_{n,k,1}+\sum_{\ell=x+1}^n a_{n,k,\ell}}
\geq \frac{\frac{2k}{n(n-1)}}{\frac{2k}{n(n-1)}
+\left(\frac{2k}{n(n-1)}\right)\left(\frac{1}{c}\right)^{x-1}}
\geq 1-c^{-x+1}.
\end{equation*}

The proof of the second part of Theorem~\ref{thm:characterization} 
uses similar techniques.
%
\end{proof}

In Figure~\ref{fig:equiv_class}, we plot $b(k,x)$ for NASBench-101,
which supports Theorem~\ref{thm:characterization}.
Next, we may ask whether the one-hot adjacency matrix encoding can be truncated.
However, even removing one bit from the adjacency matrix encoding can be very
costly, because each single edge makes the difference between a path from the input
node to the output node vs.\ no path from the input node to the output node.
In the next theorem, we show that the probability of a random graph containing
any individual edge is at least $2k/(n(n-1))$. Therefore,
truncating the adjacency matrix encoding even by a single bit results in
significant information loss.
In the following theorem, let $E_{n,k,r}$ denote the edge set of $G_{n,k,r}$.
Given $1\leq z\leq n(n-1)/2$, we slightly abuse notation
by writing $z\in E_{n,k,r}$ if the edge with index $z$ in the adjacency
matrix is in $E_{n,k,r}$.
\begin{restatable}{rethm}{adjacency}\label{thm:adjacency}
Given $n\leq k \leq \frac{n(n-1)}{2}$
and $1\leq z\leq n(n-1)/2$, we have $P(z\in E_{n,k,r})>\frac{2k}{n(n-1)}.$
\end{restatable}

\begin{proof}[\textbf{Proof}]
Recall that \emph{sample random architecture} adds each edge with probability
$2k/(n(n-1))$ and rejects in step \emph{(3)} 
if there is no path from the input to the output.
Define $G'_{n,k,r}$ and \emph{valid} as in the proof of 
Theorem~\ref{thm:characterization} and let $E'_{n,k,r}$ denote the edge set of $G'_{n,k,r}$.
Then
\begin{align*}
   \frac{P(G'_{n,k,r}\text{ is valid}\mid z\in E'_{n,k,r})}
{P(G'_{n,k,r}\text{ is valid})} =
\frac{P(z\in E'_{n,k,r} \mid G'_{n,k,r}\text{ is valid})}{P(z\in E'_{n,k,r})} > 1, 
\end{align*}
where the first equality comes from Bayes' theorem, and the inequality follows
because there is a natural bijection $\phi$ from graphs with $z$ to graphs without $z$
given by removing $z$, where $G$ is valid if $\phi(G)$ is valid but the reverse does not hold. 
Therefore,
\begin{align*}
P(z\in E_{n,k,r})&=P(z\in E'_{n,k,r}\mid G'_{n,k,r}\text{ is valid})
=\frac{P(G'_{n,k,r}\text{ is valid}\mid z\in E'_{n,k,r})P(z\in E'_{n,k,r})}
{P(G'_{n,k,r}\text{ is valid})}\\
& > P(z\in E'_{n,k,r}) = \frac{2k}{n(n-1)}. \qedhere
\end{align*}
\end{proof}

Our theoretical results show that the path encoding can be heavily truncated,
while the adjacency matrix cannot be truncated. 
In the next section, we verify this experimentally (Figure~\ref{fig:equiv_class}).

\section{Experiments}
\label{sec:experiments}

In this section, we present our experimental results.
All of our experiments follow the Best Practices for NAS 
checklist~\citep{lindauer2019best}. We discuss our adherence to these
practices in 
Appendix~\ref{app:experiments}.
In particular, we release our code at 
\url{https://github.com/naszilla/naszilla}.
We run experiments on two search spaces which we describe below.

The NASBench-101 dataset~\citep{nasbench} consists of approximately 423,000 
neural architectures pretrained on CIFAR-10.
The search space is a cell consisting of 7 nodes.
The first node is the input, and the
last node is the output. The middle five nodes can take one of
three choices of operations, and there can be at most 9 edges between
the 7 nodes.
The NASBench-201 dataset~\citep{nasbench201} consists of
$15625$ neural architectures separately trained on each of
CIFAR-10, CIFAR-100, and ImageNet16-120.
The search space consists of a cell which is a complete directed 
acyclic graph with 4 nodes.
Each edge takes an operation, and there are five possible operations.
%

We split up our first set of experiments based on the three encoding-dependent subroutines:
\emph{sample random architecture, perturb architecture}, and 
\emph{train predictor model}.
These three subroutines are the only encoding-dependent building
blocks necessary for many NAS algorithms.

\paragraph{Sample random architecture.}
Most NAS algorithms use a subroutine to draw an architecture randomly 
from the search space. Although this operation is more generally parameterized
by a distribution over the search space, it is often instantiated with the choice
of architecture encoding.
Given an encoding, we define a subroutine by sampling each feature
uniformly at random.
We also compare to sampling each \emph{architecture} 
uniformly at random 
from the search space (which does not correspond to any encoding).
Note that sampling architectures uniformly at random can be very computationally
intensive. It is much easier to sample \emph{features} uniformly at random.

\paragraph{Perturb architecture.}
Another common subroutine in NAS algorithms is to make a small change to a
given architecture.
The type of modification depends on the encoding. 
For example, a perturbation might be to
change an operation, add or remove an edge, or add or remove a path.
Given an encoding and a mutation factor $m$, 
we define a perturbation subroutine by resampling each
feature of the encoding uniformly at random with a fixed probability,
so that $m$ features are modified on average.

\paragraph{Train predictor model.}
Many families of NAS algorithms use a subroutine which learns a model based on
previously queried architectures. For example, this can take the form of
a Gaussian process within Bayesian optimization (BO), or, more recently,
a neural predictor model~\citep{shi2019multi, wen2019neural, bananas}.
In the case of a Gaussian process model, the algorithm uses a distance
metric defined on pairs of neural architectures,
which is typically chosen as the edit distance between
architecture encodings~\citep{nasbot, auto-keras}.
In the case of a neural predictor, the encodings of the queried architectures
are used as training data, and the goal is typically to predict the accuracy
of unseen architectures.

\paragraph{Experimental setup and results.}
We run multiple experiments for each encoding-dependent subroutine listed
above. Many NAS algorithms use more than one subroutine, so in each experiment,
we fix the encodings for all subroutines except for the one we are testing.
For each NAS subroutine, we experiment on algorithms that depend on
the subroutine.
In particular, for \emph{random sampling}, we run experiments on 
the Random Search algorithm. 
For \emph{perturb architecture}, we run experiments on 
regularized evolution~\citep{real2019regularized} 
and local search~\citep{white2020local}.
For \emph{train predictor model}, we run experiments on BO, testing five encodings
that define unique distance functions, as well as NASBOT~\citep{nasbot}
(which does not correspond to an encoding).
We also train a neural predictor model using seven different encodings, as
well as GCN~\citep{wen2019neural} and VAE~\citep{dvae}.
Since this runs in every iteration of a NAS 
algorithm~\cite{shi2019multi, bananas, wen2019neural},
we plot the mean absolute error on the test set for different sizes
of training data.
Finally, we run experiments on BANANAS~\citep{bananas},
varying all three subroutines at once.
We directly used the open source code for each algorithm, except we changed
the hyperparameters based on the encoding, described below. 
Details on the implementations for each algorithm are discussed in
Appendix~\ref{app:experiments}.

Existing NAS algorithms may have hyperparameters that are optimized for a
specific encoding, therefore, we perform hyperparameter tuning for each
encoding. We just need to be careful that we do not perform hyperparameter 
tuning for specific \emph{datasets} 
(in accordance with NAS best practices~\citep{lindauer2019best}). 
Therefore, we perform the hyperparameter search on CIFAR-100 from NAS-Bench-201, 
and apply the results on NAS-Bench-101. We defined a search region for each 
hyperparameter of each algorithm, and then for each encoding, we ran 50 
iterations of random search on the full hyperparameter space. 
We chose the configuration that minimizes the validation loss of the NAS 
algorithm after 200 queries. 


%

\begin{figure*}
\centering %
\includegraphics[width=0.33\textwidth]{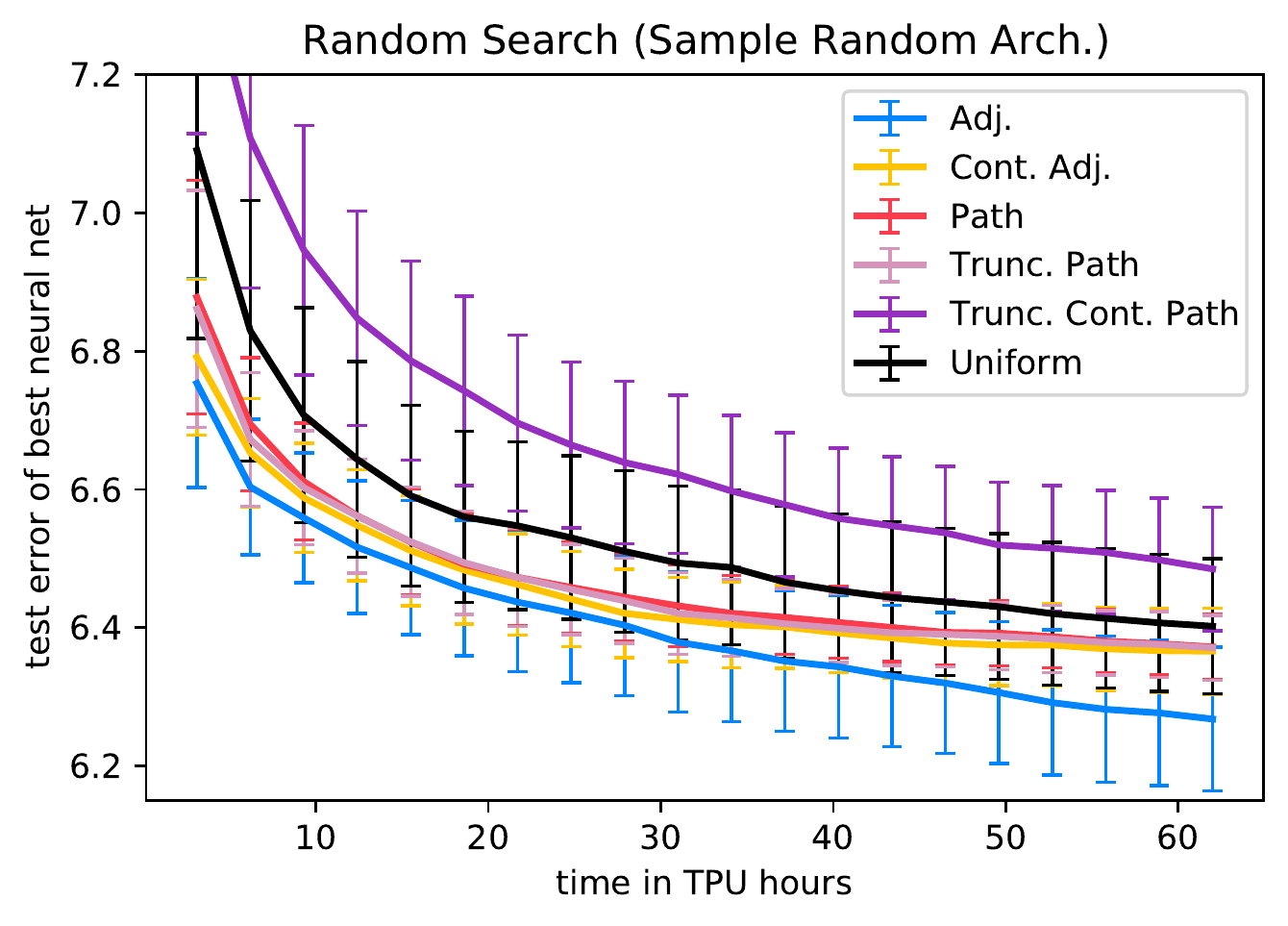}
\hspace{-3pt}
\includegraphics[width=0.33\textwidth]{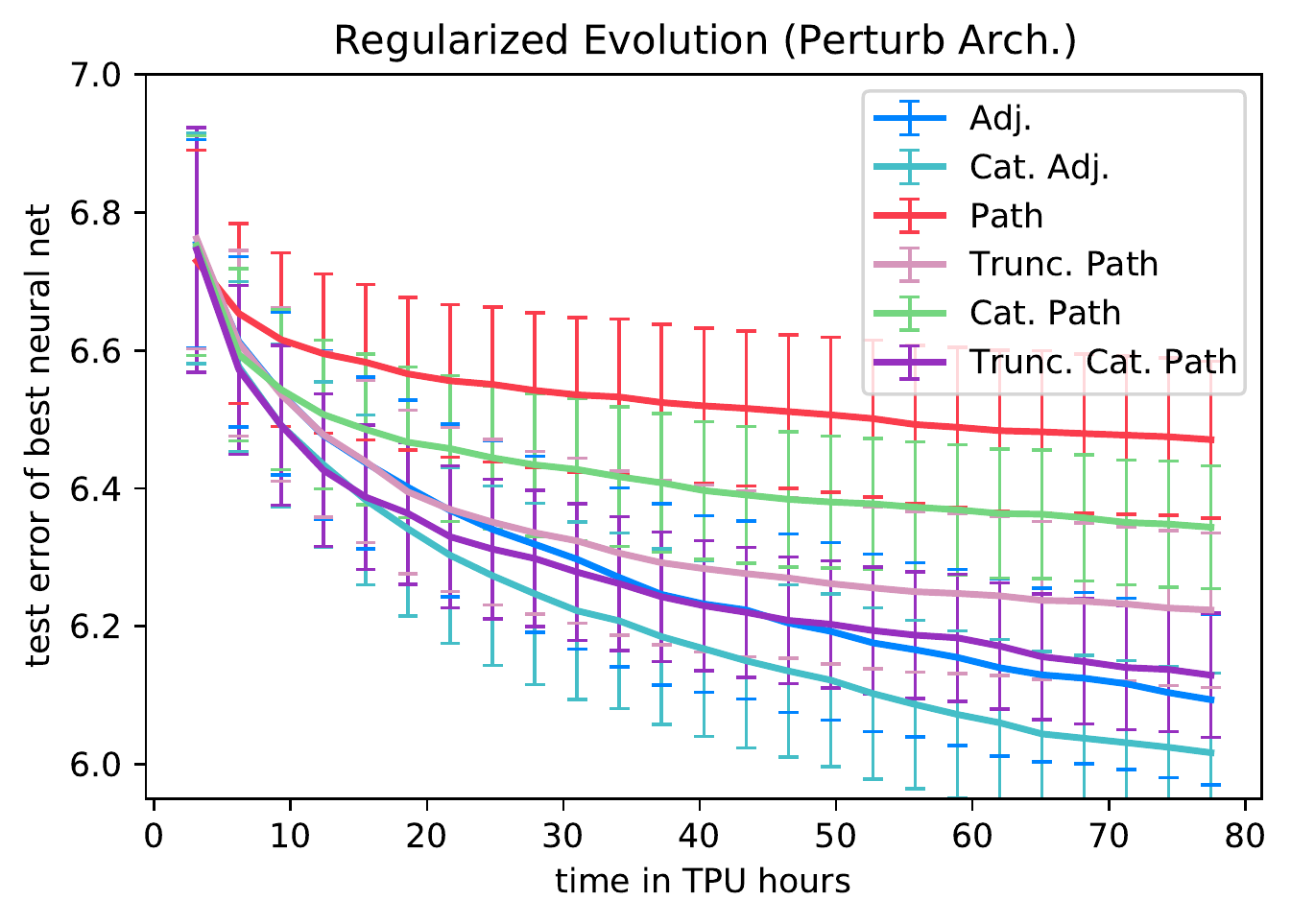}
\hspace{-3pt}
\includegraphics[width=0.32\textwidth]{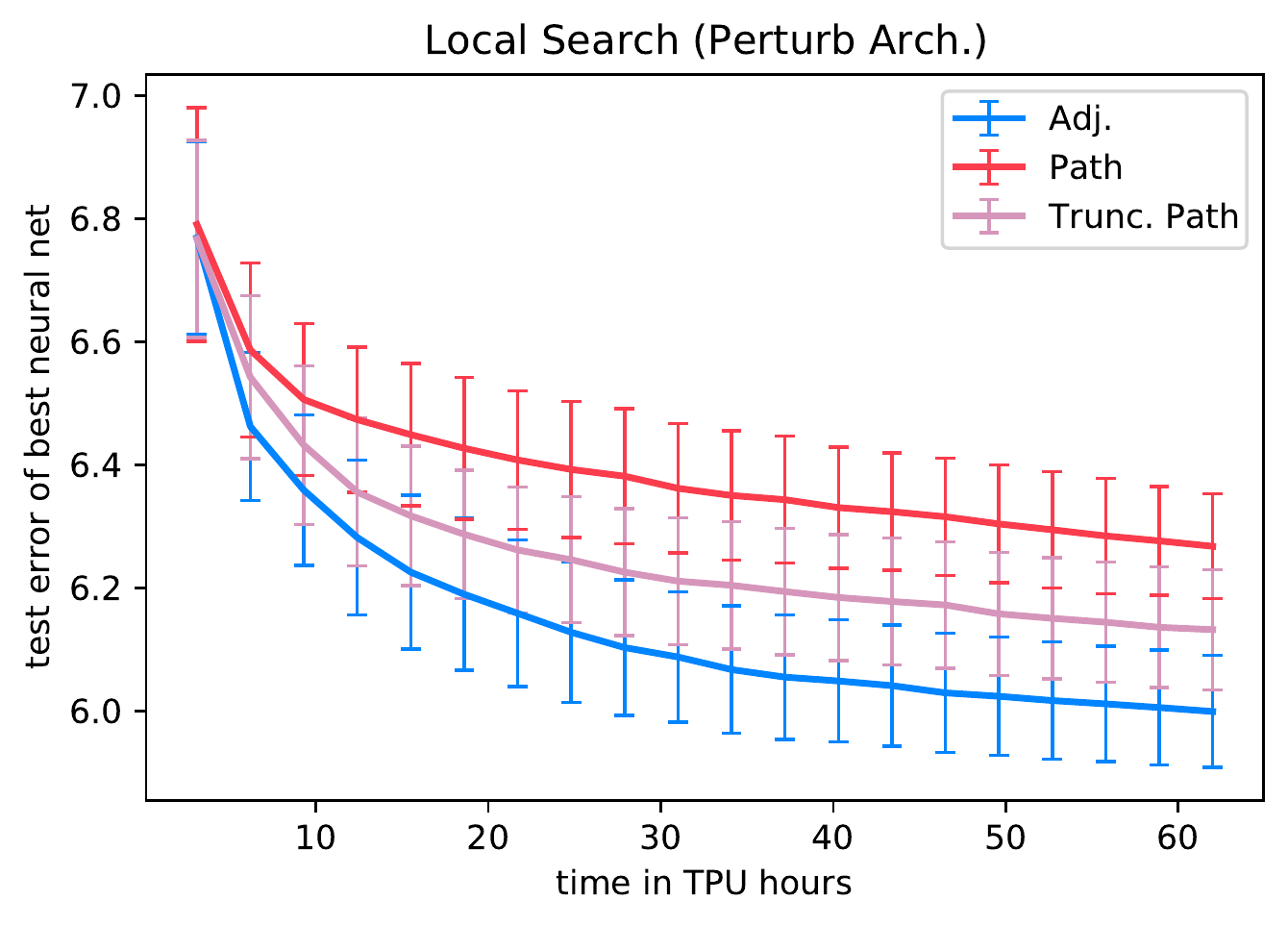}
\includegraphics[width=0.33\textwidth]{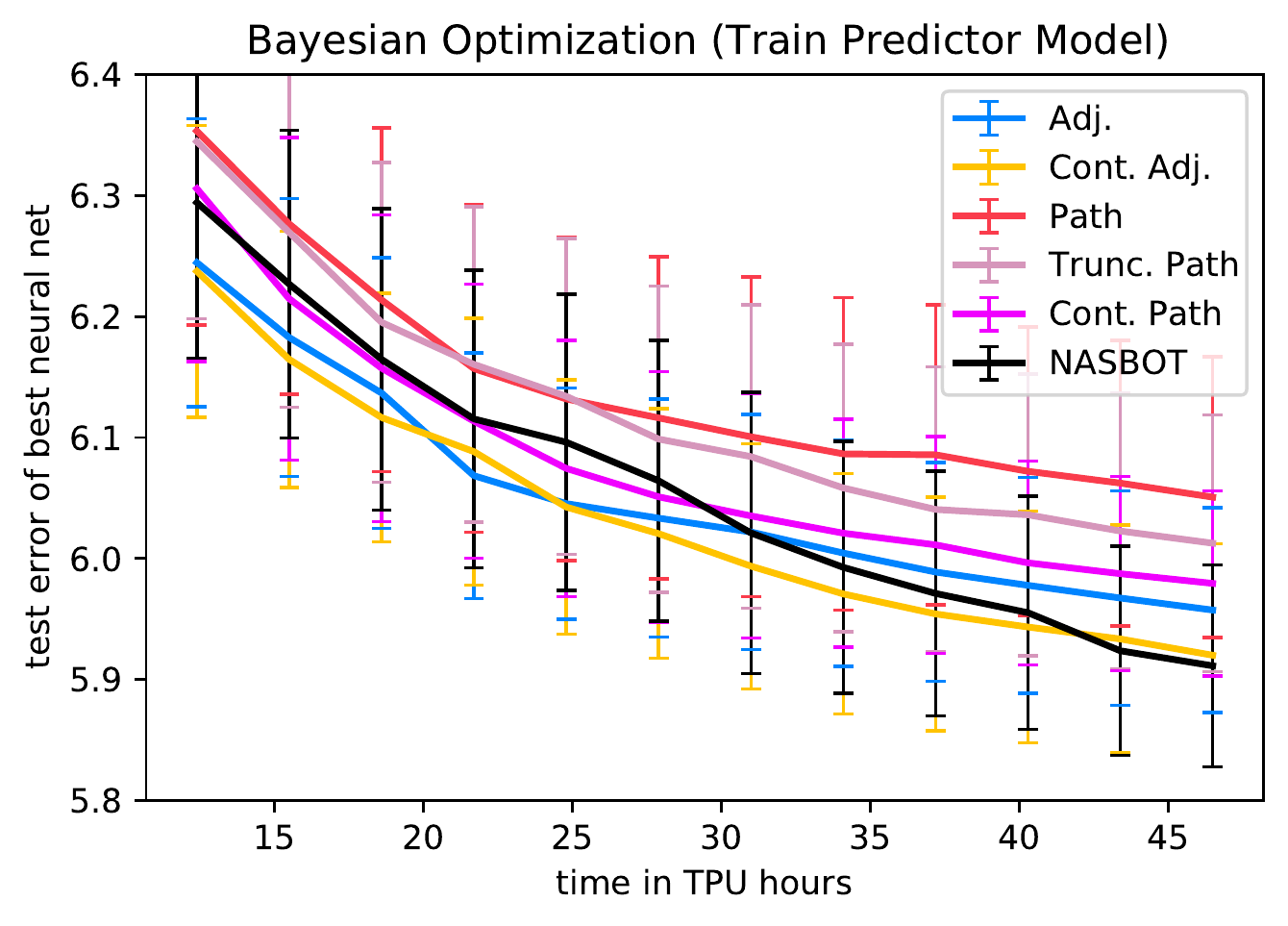}
\hspace{-3pt}
\includegraphics[width=0.33\textwidth]{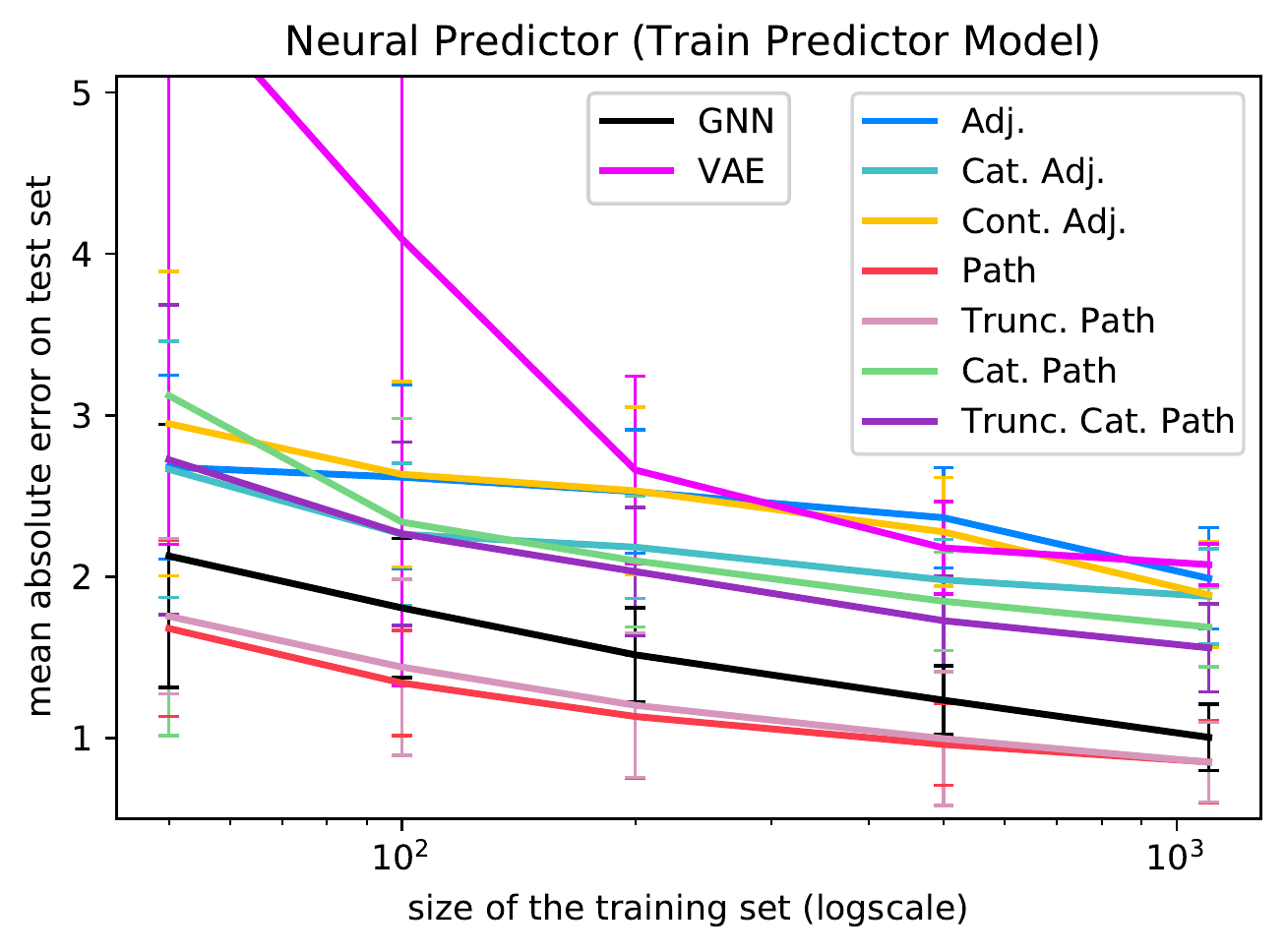}
\hspace{-3pt}
\includegraphics[width=0.32\textwidth]{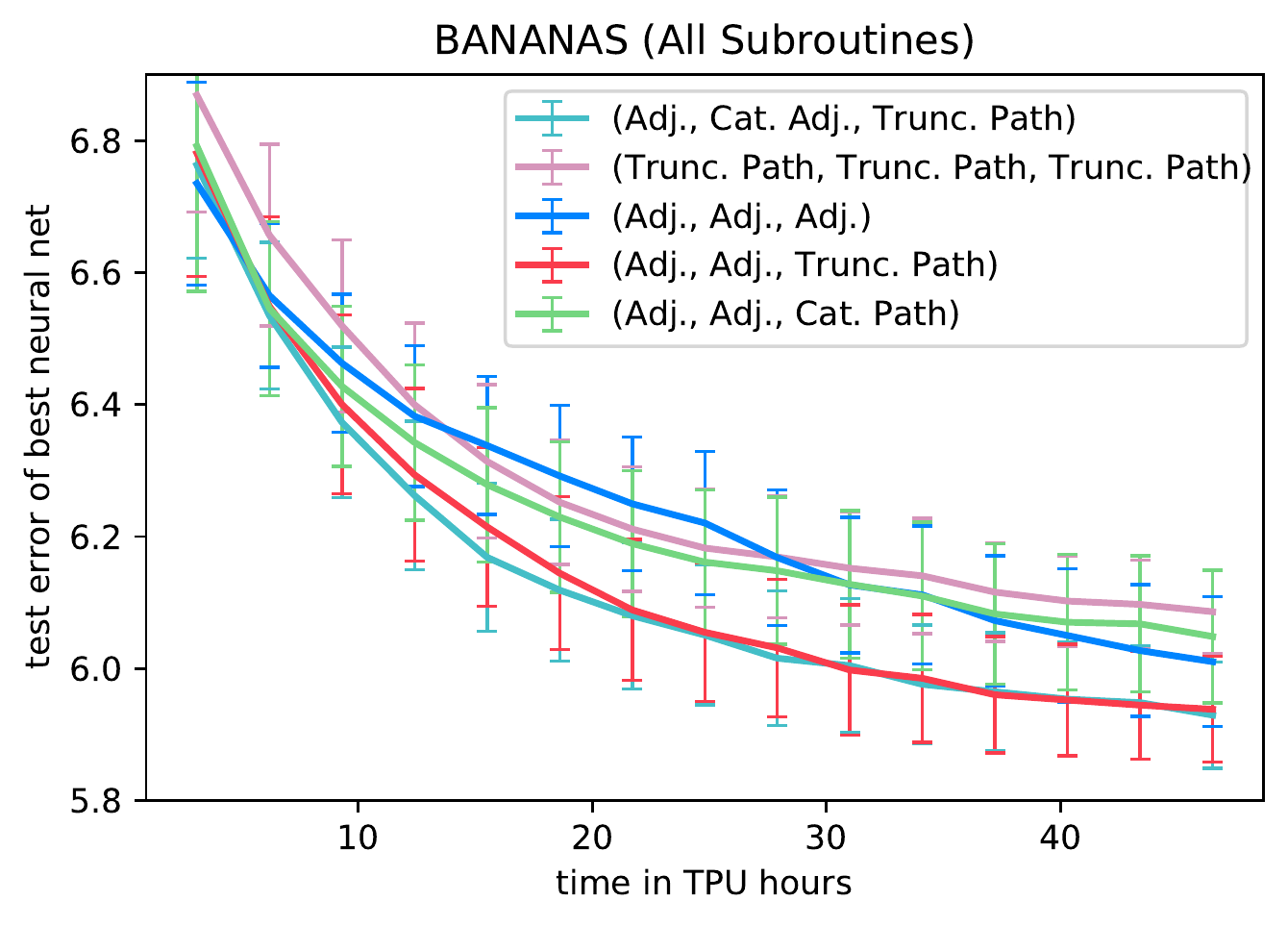}
\caption{
Experiments on NASBench-101 with different encodings, keeping all but one subroutine
fixed: \emph{random sampling} (top left), \emph{perturb architecture} (top middle, top right), \emph{train predictor model} (bottom left, bottom middle), or varying all three subroutines (bottom right).
}
\label{fig:encoding_results}
\end{figure*}

In each experiment, we report the test error of the neural network with the best
validation error after time $t$, for $t$ up to 130 TPU hours.
We run 300 trials for each algorithm and record the mean test errors.
See Figure~\ref{fig:encoding_results} for the results on NASBench-101.
We present more experiments for NASBench-201
in Appendix~\ref{app:experiments},
seeing largely the same trends.
Depending on the subroutine, two encodings might be functionally equivalent, 
which is why not all encodings appear in each experiment
(for example, in local search, there is no difference between one-hot and
categorical encodings). 
There is no overall 
best encoding; instead, each encoding has varied performance for each subroutine,
and the results in Figure~\ref{fig:encoding_results} act as a guideline for
which encodings to use in which subroutines.
As a rule of thumb, the adjacency matrix-based encodings perform well
for the sample random architecture and perturb architecture subroutines,
but the path-based encodings far outperformed the adjacency matrix-based
encodings for the train predictor model subroutines.
Categorical, one-hot, adjacency-based, path-based,
and continuous encodings are all best in certain settings.
Some of our findings explain the success of prior algorithms, e.g., regularized
evolution using the categorical adjacency encoding, and BANANAS using the path encoding
in the meta neural network.
We also show that combining the best encodings for each subroutine in BANANAS yields the best performance.
Finally, we show that the path encoding even outperforms GCNs and VAEs in
the neural predictor experiment.

In Figure~\ref{fig:encoding_results}, \emph{Trunc.\ Path} denotes the path encoding
truncated from $\sum_{i=0}^5 3^i=364$ to $\sum_{i=0}^3 3^i=40.$
As predicted by Theorem~\ref{thm:characterization}, this does not decrease performance.
In fact, in regularized evolution, the truncation improves performance significantly because
perturbing with the full path encoding is more likely to add uncommon paths that do not improve accuracy.
We also evaluate the effect of truncating the one-hot adjacency matrix encoding 
on regularized evolution, from the full 31 bits (on NASBench-101) to 0 bits,
and the path encoding from 31 bits (out of 364) to 0 bits.
See Figure~\ref{fig:equiv_class}.
The path encoding is much more robust to truncation, consistent with 
Theorems~\ref{thm:characterization} and~\ref{thm:adjacency}.

\paragraph{Outside search space experiment.}
In the set of experiments above, we tested the effect of encodings on a neural
predictor model by computing the mean absolute error between
the predicted vs.\ actual errors on the test set,
and also by evaluating the performance of BANANAS when changing the encoding of
its neural predictor model.
The latter experiment tests the predictor model's ability to predict
the \emph{best} architectures, not just all architectures on average.
We take this one step further and test the ability of the neural predictor
to generalize beyond the search space on which it was trained.
We set up the experiment as follows. 
We define the training search space as a subset of NASBench-101:
architectures with at most 6 nodes and 7 edges.
We define the disjoint test search space as architectures
with 6 nodes and 7 to 9 edges.
The neural predictor is trained on 1000 architectures
and predicts the validation loss of the 5000 architectures from the test search space.
We evaluate the losses of the ten architectures with the highest predicted validation loss.
We run 200 trials for each encoding and average the results.
See Table~\ref{tab:outside_ss}.
The adjacency encoding performed the best. 
An explanation is that for the path encoding, there are features
(paths) in architectures from the test set that do not exist in the training set.
This is not the case for the adjacency encoding: all features (edges)
from architectures in the test set have shown up in the training set.

\begin{table*}[t]
\caption{Ability of neural predictor with different encodings to generalize 
beyond the search space.}
\setlength\tabcolsep{0pt}
\begin{tabular*}{\textwidth}{l @{\extracolsep{\fill}}*{8}{S[table-format=1.4]}} 
\toprule
\multicolumn{1}{c}{Encoding} & \multicolumn{2}{c}{Validation error} & 
\multicolumn{2}{c}{Test error} \\
\cmidrule{2-5} & {Top 10 avg.} & {Top 1 avg.} & {Top 10 avg.} & {Top 1 avg.} &  \\ 
\midrule
 Adjacency & \hspace{3.5mm}\textbf{5.888} & \hspace{2.5mm}\textbf{5.505} & \hspace{3.5mm}\textbf{6.454} & \hspace{2.5mm}\textbf{6.056} \\
 Categorical Adjacency & 7.589 & 6.191 & 8.155 & 7.086\\
 \midrule
 Path & 5.967 & 5.606 & 6.616 & 6.335\\
 Truncated Path & 6.082 & 5.644 & 6.712 & 6.452\\
 Categorical Path & 6.357 & 5.703 & 6.939 & 6.489\\
 Truncated Categorical Path & 6.339 & 5.895 & 6.918 & 6.766\\
\bottomrule
\end{tabular*} 
\label{tab:outside_ss}
\end{table*}

\paragraph{Equivalence class experiments.}
Recall that the path encoding function $e$ is not one-to-one 
(see Figure~\ref{fig:non_onto}).
In general, this is not desirable because information is lost when two architectures map to the same encoding. 
However, if the encoding function only maps architectures with similar accuracies to the same encoding, then the behavior is beneficial.
On the NASBench-101 dataset, we compute the path encoding of all 423k
architectures, and then we compute the average
standard deviation of accuracies among architectures with the same encoding
(i.e., we look at the standard deviations within the equivalence classes defined by the encoding).
See Figure~\ref{fig:equiv_class}.
The result is an average standard
deviation of 0.353\%, compared to the 5.55\% standard deviation over the 
entire set of architectures. 

\begin{figure*}
\centering %
\includegraphics[width=0.27\textwidth]{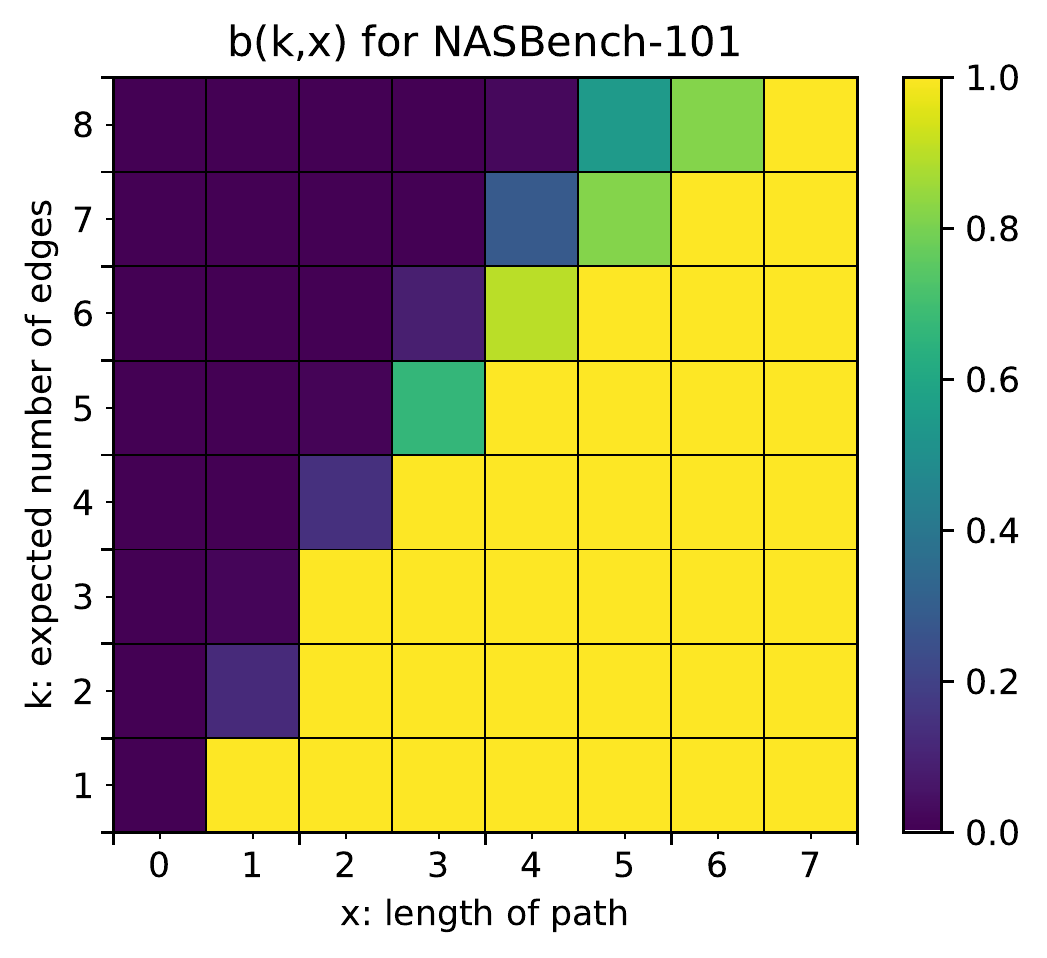}
\hspace{-3pt}
\includegraphics[width=0.35\textwidth]{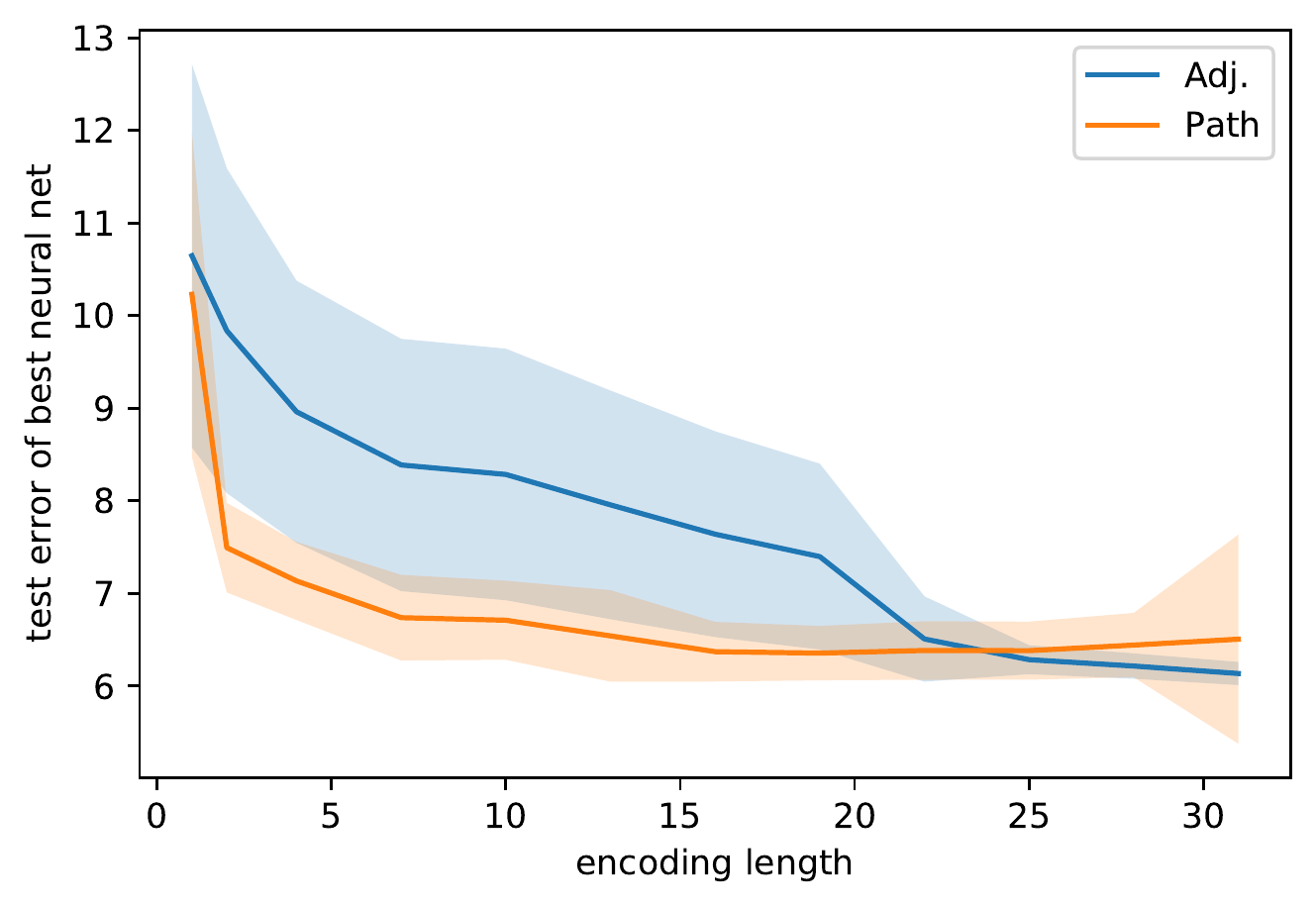}
\hspace{-3pt}
\includegraphics[width=0.35\textwidth]{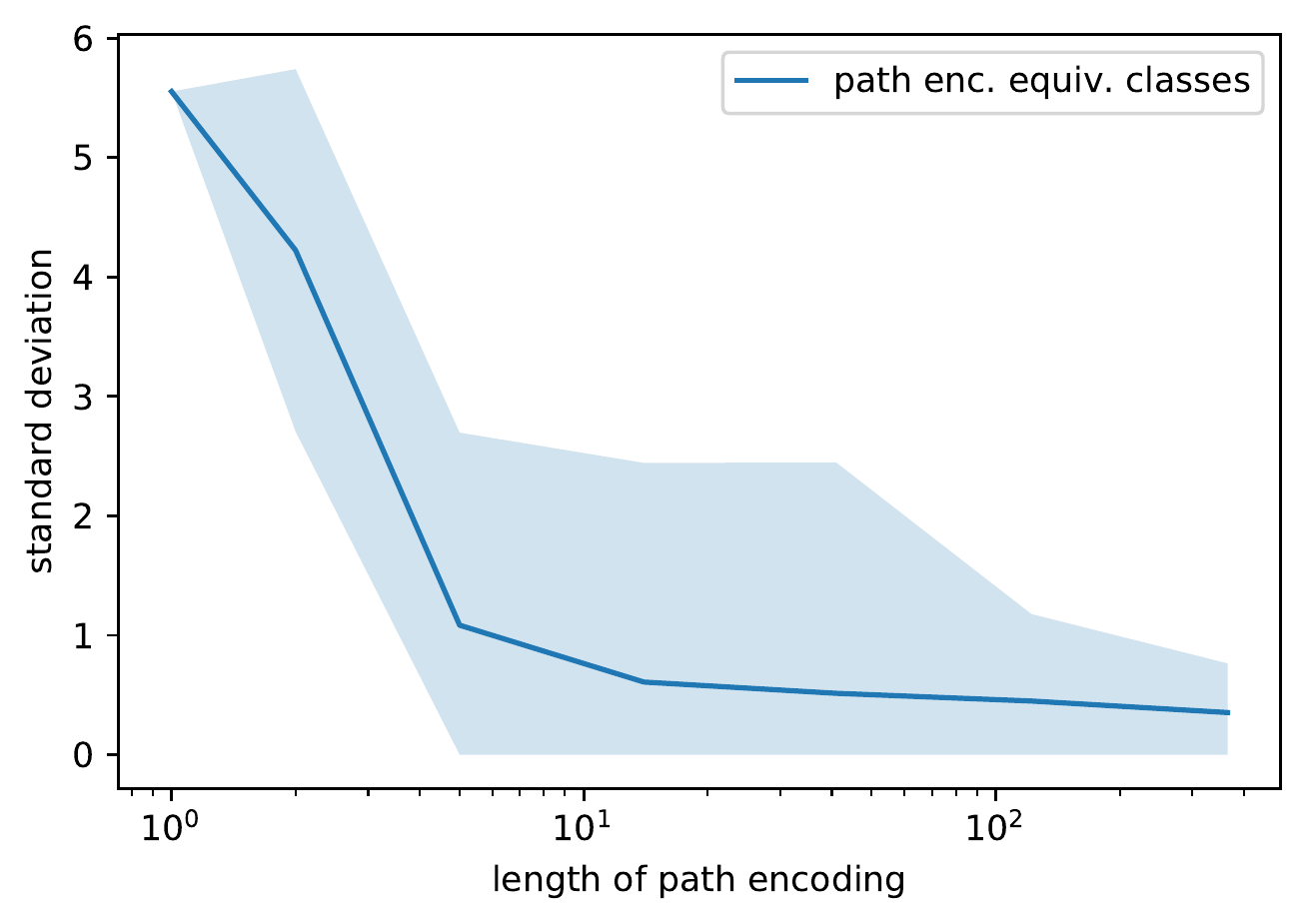}
\caption{
Plot of $b(k,x)$ on NASBench-101 (left), which is consistent 
with Theorem~\ref{thm:characterization}.
Truncation of encodings for regularized evolution on NASBench-101 (middle).
Average standard deviation of accuracies within each equivalence class defined by the path encoding at different
levels of truncation on NASBench-101 (right).
}
\label{fig:equiv_class}
\end{figure*}

\section{Conclusion}
\label{sec:conclusion}

In this paper, we give the first formal study of encoding schemes for neural 
architecture search.
We define eight different encodings and characterize the scalability of each one.
We then identify three encoding-dependent subroutines used by NAS algorithms---\emph{sample
random architecture}, \emph{perturb architecture}, and \emph{train predictor model}---and
we run experiments to find the best encoding 
for each subroutine in many popular algorithms.
We also conduct experiments on the ability of a neural predictor to generalize beyond
the training search space, given each encoding.
Our experimental results allow us to disentangle the algorithmic and encoding-based
contributions of prior work, and act as a set of guidelines for which encodings to
use in future work.
Overall, we show that encodings are an important, nontrivial design decision in the
field of NAS. Designing and testing new encodings is an exciting next step.

\section{Broader Impact} \label{sec:impact}
Our work gives a study on encodings for neural architecture search, with the goal
of helping future researchers improve their NAS algorithms.
Therefore, this work may not have a direct impact on society,
since it is two levels of abstraction from real applications,
but it can indirectly impact society.
As an example, our work may inspire the creation of a new state-of-the-art NAS
algorithm, which is then used to improve the performance of various deep learning
algorithms, which can have both beneficial and detrimental uses (e.g. optimizers that reduce $\text{CO}_2$ emissions, or deep fake generators).
Due to the recent push for the AI community to be more
conscious and prescient about the societal impact of its work~\cite{hecht2018time},
we are hoping that future AI models, including ones influenced by our work, will have
a positive impact on society.

\section{Acknowledgments}
We thank the anonymous reviewers for their helpful suggestions. WN was supported by U.S. Department of Energy Office of Science under Contract No. DE-AC02-76SF00515.

\bibliography{main}
\bibliographystyle{plain}

\newpage
\appendix
\section{Details from Section~\ref{sec:prelim} (Encodings for NAS) \label{app:prelim}}

We give the details from Section~\ref{sec:prelim}.
We restate the random graph model here more formally.
\begin{restatable}{redef}{randomgraph}\label{def:random_graph}
Given nonzero integers $n, r$, and $k<\nicefrac{n(n-1)}{2}$,
a random graph $G_{n, k, r}$ is generated as follows: \\
(1) Denote $n$ nodes by 1 to $n$ and label each node randomly with one of $r$ operations.\\
(2) For all $i<j$, add edge $(i,j)$ with probability 
$\frac{2k}{n(n-1)}$.\\
(3) If there is no path from node 1 to node $n$, \texttt{goto} \texttt{(1)}.
\end{restatable}

Let $G'_{n, k, r}$ denote the random graph outputted by
the above procedure without step (3).
Since the number of pairs $(i,j)$ such that $i<j$ is 
$\frac{n(n-1)}{2}$, the expected number of edges of $G'_{n, k, r}$ is $k$.
Define $a_{n,k,\ell}$ as the expected number of paths
from node 1 to node $n$ of length $\ell$ in $G'_{n,k,r}$.
Formally, we set $\P=\{\text{paths from node 1 to }n\text{ in }G'_{n,k,r}\}$,
and define
\begin{equation*}
a_{n,k,\ell}=\mathbb{E}\left[\left|p\in\mathcal{P}\right|\mid |p|=\ell\right].
\end{equation*}

Recall that 
\begin{equation*}
    b(k, x)=\frac{\sum_{\ell=1}^x a_{n,k,\ell}}{\sum_{\ell=1}^n a_{n,k,\ell}}.
\end{equation*}

Recall that Theorem~\ref{thm:characterization} gives a 
full characterization of $b(k,x)$ in terms of
$k$ and $n$, up to constant factors. 
We prove there exists a phase transition
for $b(k,x)$ at $x=\frac{k}{n}.$
We restate this theorem for convenience.

\characterization*

As noted by prior work~\cite{bananas}, 
there are two caveats when applying this type of theorem to NAS performance.
The theorem considers the distribution from Definition~\ref{def:random_graph},
not the distribution of architectures encountered in a real search,
and the most common paths in the distribution are not necessarily the ones
with the most entropy in predicting whether an architecture has a high accuracy.
However, two prior works have experimentally showed that truncating the path
encoding does not decrease performance~\cite{npenas, bananas},
and we gave even more experimental evidence in Section~\ref{sec:experiments}.

To prove Theorem~\ref{thm:characterization},
we use the following well-known
bounds on binomial coefficients (e.g., \cite{stanica2001good}).

\begin{theorem} \label{thm:binomial}
Given $0\leq \ell\leq n$, 
\begin{equation*}
    \left(\frac{n}{\ell}\right)^\ell \leq \binom{n}{\ell}
    \leq \left(\frac{en}{\ell}\right)^\ell.
\end{equation*}
\end{theorem}

Now we give upper and lower bounds on $a_{n,k,\ell}$ which will be used
for the rest of the proofs.
The next fact is similar to Lemma C.3 from BANANAS~\cite{bananas}.

\begin{fact}\label{fact:binom}
Given $n\leq k \leq \frac{n(n-1)}{2}$, and $0<x<n$,
we have
\begin{equation*}
\frac{2k}{n(n-1)}\left(\frac{2k(n-2)}{(\ell-1)n(n-1)}\right)^{\ell-1}
\leq a_{n,k,\ell}\leq
\frac{2k}{n(n-1)}\left(\frac{2ek(n-2)}{(\ell-1)n(n-1)}\right)^{\ell-1}
\end{equation*}
\end{fact}

\begin{proof}
First, we have
\begin{equation*}
a_{n,k,\ell}=\binom{n-2}{\ell-1}\left(\frac{2k}{n(n-1)}\right)^\ell
\end{equation*}
because on a path from node 1 to node $n$ with length $\ell$,
there are $\binom{n-2}{\ell-1}$ choices of intermediate
nodes from 1 to $n$. Once the nodes are chosen, we need all $\ell$
edges between the nodes to exist, and each edge exists independently 
with probability $\frac{2}{n(n-1)}\cdot k.$
Then we achieve the desired result 
by applying Theorem~\ref{thm:binomial}.
\end{proof}

Now we prove the lower bound of Theorem~\ref{thm:characterization}.

\begin{lemma} \label{lem:upper}
Given $n\leq k \leq \frac{n(n-1)}{2}$ and $c>2$,
for $x> \frac{2eck}{n}$, $b(k,x)>1-c^{-x+1}$.
\end{lemma}

\begin{proof}
Given $n\leq k \leq \frac{n(n-1)}{2}$ and $x> \frac{2eck}{n}$,
we give a lower bound for $\sum_{\ell=1}^x a_{n,k,\ell}$
and an upper bound for $\sum_{\ell=x+1}^n a_{n,k,\ell}$.

When $\ell=1$, we have $\binom{n-2}{\ell-1}=1$.
Therefore,
\begin{equation*}
\sum_{\ell=1}^x a_{n,k,\ell}\geq 
a_{n,k,1}=\frac{2k}{n(n-1)}.
\end{equation*}

Now we upper bound  $\sum_{\ell=x+1}^n a_{n,k,\ell}$.

\begin{align}
\sum_{\ell=x+1}^n a_{n,k,\ell}
&\leq \sum_{\ell=x+1}^n \frac{2k}{n(n-1)}
\left(\frac{2ek(n-2)}{(\ell-1)n(n-1)}\right)^{\ell-1}\nonumber \\
&= \frac{2k}{n(n-1)}\sum_{\ell=x+1}^n \left(\frac{2ek(n-2)}{(\ell-1)n(n-1)}\right)^{\ell-1}\nonumber\\
&\leq \frac{2k}{n(n-1)}\sum_{\ell=x+1}^n \left(\frac{1}{c}\right)^{\ell-1}\label{eq:upper_sum}\\
&\leq \left(\frac{2k}{n(n-1)}\right)
\left(\frac{1}{c}\right)^x\sum_{\ell=0}^{\infty}\left(\frac{1}{c}\right)^\ell \nonumber\\
&= \left(\frac{2k}{n(n-1)}\right)
\left(\frac{1}{c}\right)^x\left(\frac{1}{1-\frac{1}{c}}\right)\nonumber\\
&= \left(\frac{2k}{n(n-1)}\right)
\left(\frac{1}{c}\right)^x\left(\frac{c}{c-1}\right)\nonumber\\
&< \left(\frac{2k}{n(n-1)}\right)
\left(\frac{1}{c}\right)^{x-1}\label{eq:upper_c}
\end{align}

In inequality~\ref{eq:upper_sum}, we use the fact that for all $\ell\geq x+1$,
\begin{equation*}
\ell\geq x+1> \frac{2eck}{n}+1\implies
\frac{2ek(n-2)}{(\ell-1)n(n-1)}\leq\frac{2ek}{(\ell-1)n}\leq \frac{1}{c}
\end{equation*}
and in inequality~\ref{eq:upper_c}, we use the fact that $c>2$.

Therefore, we have
\begin{align}
b(k,x)&=\frac{\sum_{\ell=1}^x a_{n,k,\ell}}{\sum_{\ell=1}^n a_{n,k,\ell}}\nonumber\\
&=\frac{\sum_{\ell=1}^x a_{n,k,\ell}}{\sum_{\ell=1}^x a_{n,k,\ell}
+\sum_{\ell=x+1}^n a_{n,k,\ell}}\nonumber\\
&\geq \frac{\frac{2k}{n(n-1)}}{\frac{2k}{n(n-1)}
+\left(\frac{2k}{n(n-1)}\right)\left(\frac{1}{c}\right)^{x-1}}\label{eq:replace}\\
&= \frac{1}{1+\left(\frac{1}{c}\right)^{x-1}}\nonumber\\
&\geq 1-c^{-x+1}.\nonumber
\end{align}
In inequality~\ref{eq:replace}, we use the fact that for all $0\leq A, B, C$, 
we know that $A\geq B$ implies $\frac{A}{A+C}\geq\frac{B}{B+C}$.

\end{proof}

Now we prove the upper bound for Theorem~\ref{thm:characterization}.

\begin{lemma} \label{lem:lower}
Given $10\leq n\leq k \leq \frac{n(n-1)}{2}$ and $c>3$,
for $x< \frac{k}{2ecn}$, $b(k,x)<2^{-\frac{k}{2n}}$.
\end{lemma}

\begin{proof}
Given $n\leq k \leq \frac{n(n-1)}{2}$ and $x< \frac{k}{2ecn}$,
now we give an upper bound for $\sum_{\ell=1}^x a_{n,k,\ell}$
and a lower bound for $\sum_{\ell=1}^n a_{n,k,\ell}$.

First we make the following claim.
For all $1\leq\ell\leq x<\frac{k}{2ecn}$, we have
\begin{equation}
\left(\frac{2ek(n-2)}{(\ell-1)n(n-1)}\right)^{\ell-1} < 
\left(\frac{4e^2 c(n-2)}{n-1}\right)^{\frac{k}{2ecn}}.
\end{equation}

Now we prove the claim. In the following inequalities, we take $\log$ to have base 2.
\begin{align}
\left(\frac{2ek(n-2)}{(\ell-1)n(n-1)}\right)^{\ell-1}
&= 2^{(\ell-1)\log\left(\frac{2ek(n-2)}{(\ell-1)n(n-1)}\right)}\nonumber\\
&=2^{(\ell-1)\log\frac{1}{\ell-1}+(\ell-1)\log\left(\frac{2ek(n-2)}{n(n-1)}\right)}\nonumber\\
&\leq 2^{\frac{k}{2ecn}\log\left(\frac{2ecn}{k}\right)
+\frac{k}{2ecn}\log\left(\frac{k}{2ecn}\cdot\frac{4e^2c(n-2)}{n-1}\right)}\label{eq:xlogx}\\
&= 2^{\frac{k}{2ecn}\left(\log\left(\frac{2ecn}{k}\right)+\log\left(\frac{k}{2ecn}\right)
+\log\left(\frac{4e^2c(n-2)}{n-1}\right)\right)}\nonumber\\
&= 2^{\frac{k}{2ecn}\log\left(\frac{4e^2c(n-2)}{n-1}\right)}\nonumber\\
&=\left(\frac{4e^2c(n-2)}{n-1}\right)^{\frac{k}{2ecn}}\nonumber
\end{align}

In inequality~\ref{eq:xlogx}, we use the fact that for any $1\leq A\leq B$, 
$A\log\left(\frac{1}{A}\right)\leq B\log\left(\frac{1}{B}\right)$
(specifically, we used $A=\ell-1$ and $B=\frac{k}{2ecn}$).

Now we have

\begin{align*}
\sum_{\ell=1}^x a_{n,k,\ell} &\leq \sum_{\ell=1}^x 
\frac{2k}{n(n-1)}\left(\frac{2ek(n-2)}{(\ell-1)n(n-1)}\right)^{\ell-1} \\
&= \frac{2k}{n(n-1)} \sum_{\ell=1}^x 
\left(\frac{2ek(n-2)}{(\ell-1)n(n-1)}\right)^{\ell-1}\\
&\leq \frac{2k}{n(n-1)} \sum_{\ell=1}^x 
\left(\frac{4e^2 c(n-2)}{n-1}\right)^{\frac{k}{2ecn}}\\
&\leq \left(\frac{2k}{n(n-1)}\right)\cdot x\cdot 
\left(\frac{4e^2 c(n-2)}{n-1}\right)^{\frac{k}{2ecn}}.
\end{align*}

Now we give the lower bound for the other summation which goes from
$\ell=1$ to $n$.
We lower bound the whole summation by a single term of the summation,
$\ell=\frac{k(n-2)}{n(n-1)}+1$.
Recall that $k\leq \frac{n(n-1)}{2}$, which implies 
$\frac{k}{n}\leq \frac{n-1}{2}<n$, 
so $\ell=\frac{k(n-2)}{n(n-1)}+1$ is indeed between $1$ and $n$.

\begin{align*}
\sum_{\ell=1}^n a_{n,k,\ell} &= \sum_{\ell=1}^n 
\frac{2k}{n(n-1)}\left(\frac{2k(n-2)}{(\ell-1)n(n-1)}\right)^{\ell-1}\\
&\geq \frac{2k}{n(n-1)} \sum_{\ell=\frac{k(n-2)}{n(n-1)}+1}^{\frac{k(n-2)}{n(n-1)}+1}
\left(\frac{2k(n-2)}{(\ell-1)n(n-1)}\right)^{\ell-1}\\
&= \frac{2k}{n(n-1)}\left(2\right)^{\frac{k(n-2)}{n(n-1)}}
\end{align*}

Therefore, 
\begin{align*}
b(k,x)&=\frac{\sum_{\ell=1}^x a_{n,k,\ell}}{\sum_{\ell=1}^n a_{n,k,\ell}}\\
&\leq \frac{\left(\frac{2k}{n(n-1)}\right)\cdot x\cdot 
\left(\frac{4e^2 c(n-2)}{n-1}\right)^{\frac{k}{2ecn}}}
{\frac{2k}{n(n-1)}\left(2\right)^{\frac{k(n-2)}{n(n-1)}}}\\
&\leq x\cdot \left(4e^2c\right)^{\frac{k}{2ecn}}\cdot\left(2\right)^{-\frac{k(n-2)}{n(n-1)}}\\
&\leq 2^{\log x}\cdot \left(2\right)^{\log\left(4e^2c\right)\cdot \frac{k}{2ecn}}
\cdot \left(2\right)^{-\frac{k(n-2)}{n(n-1)}}\\
&\leq 2^{\log\left(\frac{k}{2ecn}\right)+\frac{k}{2ecn}\log\left(4e^2c\right)
-\frac{k(n-2)}{n(n-1)}}\\
&= 2^{-\frac{k}{n}\left(\frac{n-2}{n-1}-\frac{1}{2ec}\log\left(4e^2c\right)\right)
-\log\left(\frac{k}{2ecn}\right)}\\
&\leq 2^{-\frac{k}{2n}}
\end{align*}
The final inequality holds because $n\geq 10$, $c>3$, and $\frac{k}{n}\geq 1$.

\end{proof}

The proof of Theorem~\ref{thm:characterization} follows immediately by
combining Lemmas~\ref{lem:upper} and~\ref{lem:lower}.
\section{Details from Section~\ref{sec:experiments} (Experiments)}\label{app:experiments}

In this section, we give more details from Section~\ref{sec:experiments}, and we give more experiments on NASBench-201.
First we describe the algorithms used in the experiments in Section~\ref{sec:experiments}.
\begin{itemize}
    \item Random Search consists of randomly choosing architectures and then training them, until
    the runtime budget is exceeded.
    \item Regularized evolution~\citep{real2019regularized} consists of maintaining a population
    of neural architectures. In each iteration, a subset is selected and the best architecture from the subset is mutated. The mutation replaces the oldest architecture from the population. 
    We used a population size of 30. We also found that replacing the worst architecture
    (not the oldest) performed better, so we used this version.
    \item Local search~\citep{white2020local} is a simple greedy algorithm that has only 
    recently been
    applied to NAS. We use the simplest instantiation (often called the hill-climbing algorithm).
    \item Bayesian optimization (BO) is a strong method for zeroth order optimization. We use 
    the ProBO~\citep{neiswanger2019probo} implementation, which uses a Gaussian process kernel
    and expected improvement as the acquisition function.
    \item NASBOT~\cite{nasbot} is a BO-based NAS algorithm. It was not originally
    defined for cell-based search
    spaces, so we use a variant that works for cell-based spaces~\citep{bananas}.
    \item BANANAS~\citep{bananas} is a BO-based method which uses a neural predictor model.
\end{itemize}

\subsection{Experiments on NASBench-201}

In this section, we give similar experiments to Figure~\ref{fig:encoding_results}, but with
NASBench-201 instead of NASBench-101.
Note that NASBench-201 is not as good for encoding experiments because every single architecture
has the same graph structure - a clique of size 4. The only differences are the operations.
Therefore, many encodings are functionally equivalent. 
For example, the one-hot, categorical, and continuous adjacency matrix encodings are all identical 
because the only difference is the way they encode the adjacency matrix.
I.e., these encodings will all look like a set of operations, plus some adjacency matrix encoding
that is the same for every architecture in the search space.
The one-hot adjacency matrix encoding, path encoding, and truncated path encoding are all
distinct from one another, so we run experiments with these encodings.
See Figure~\ref{fig:201_results}.
We see largely the same trends as in NASBench-101 (Figure~\ref{fig:encoding_results}).
Note that on the ImageNet-16-120 dataset, some algorithms such as NASBOT 
overfit to the training set, causing performance to decline over time.

\begin{figure*}
\centering %
\includegraphics[width=0.33\textwidth]{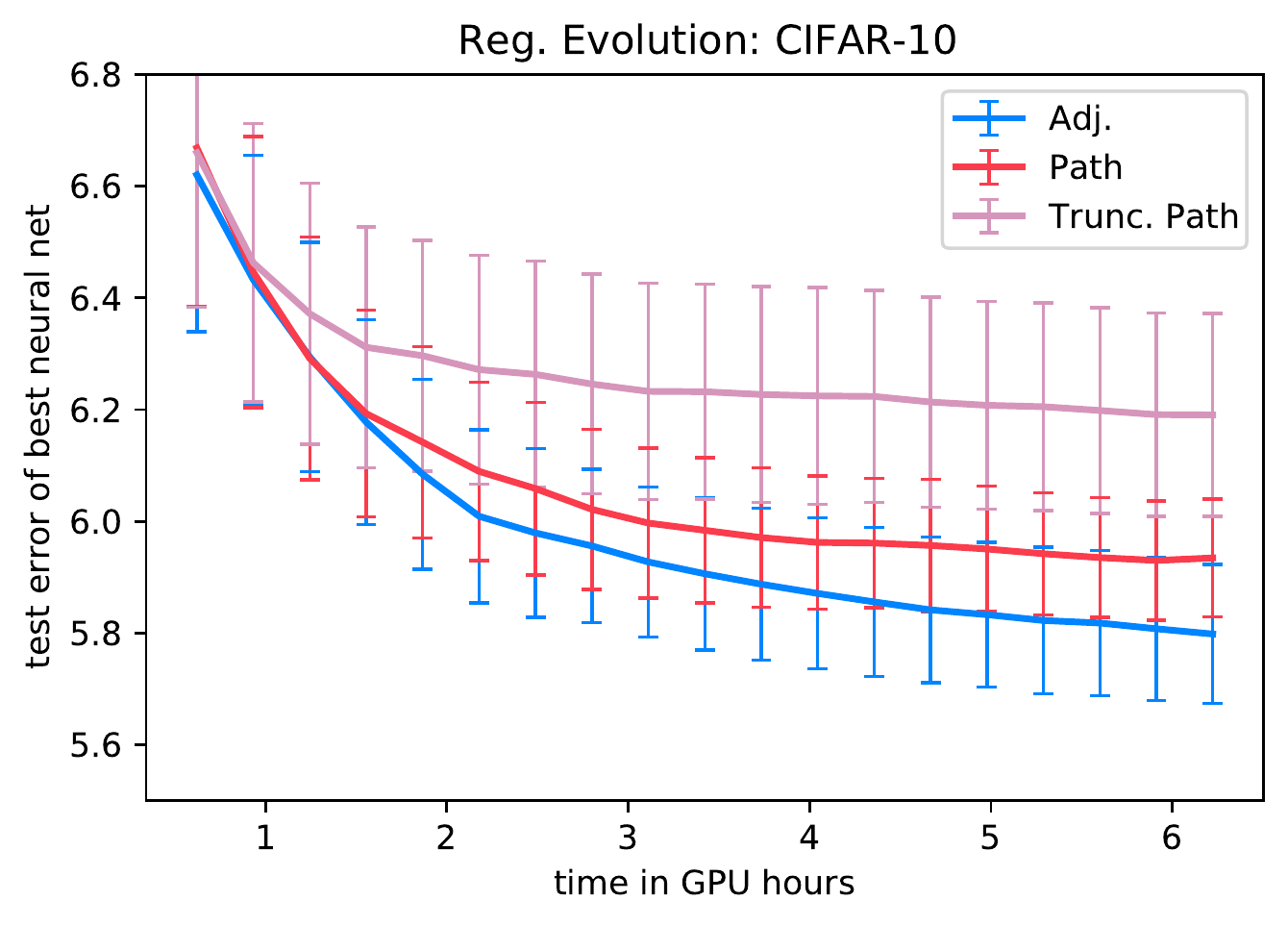}
\hspace{-3pt}
\includegraphics[width=0.33\textwidth]{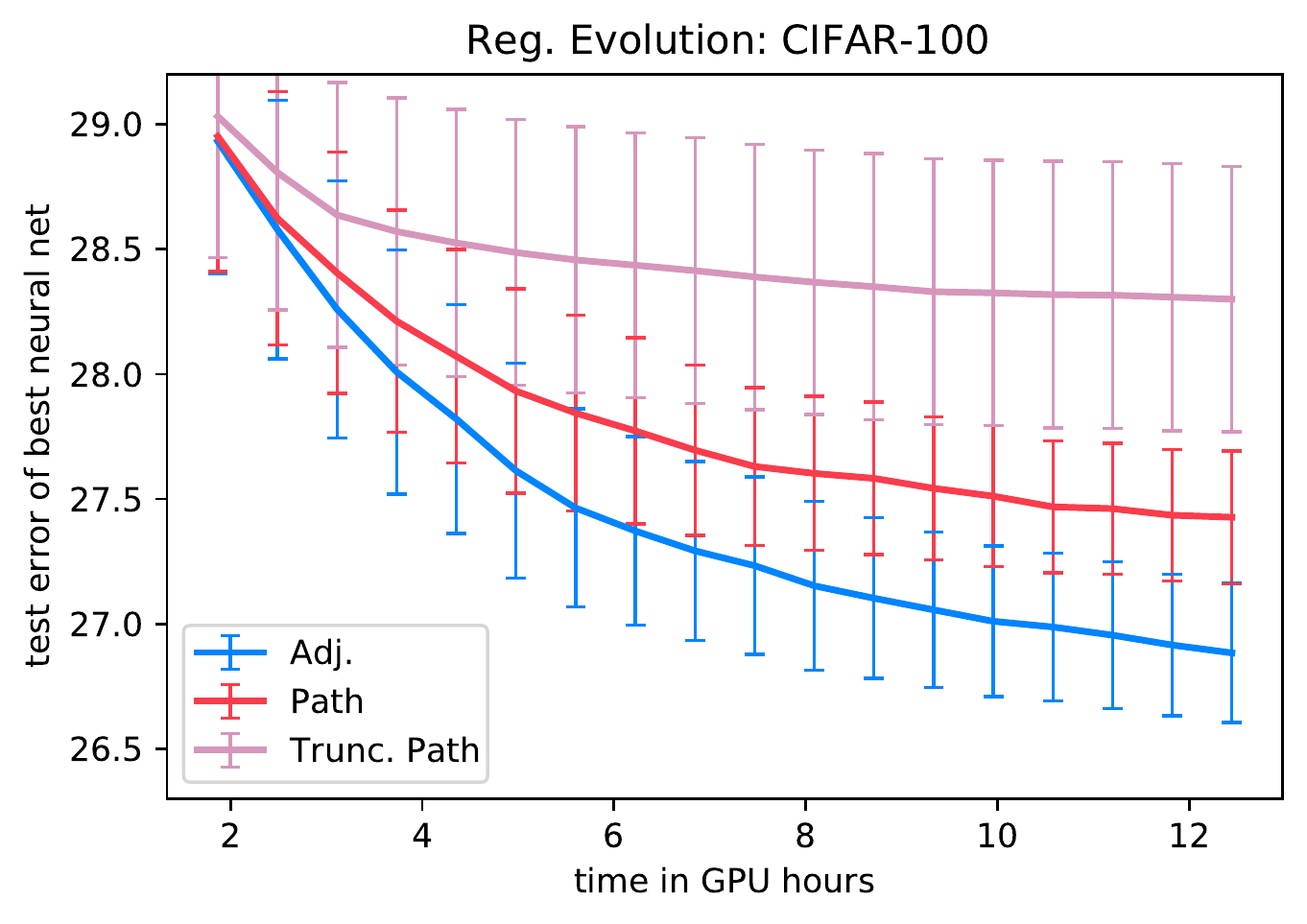}
\hspace{-3pt}
\includegraphics[width=0.32\textwidth]{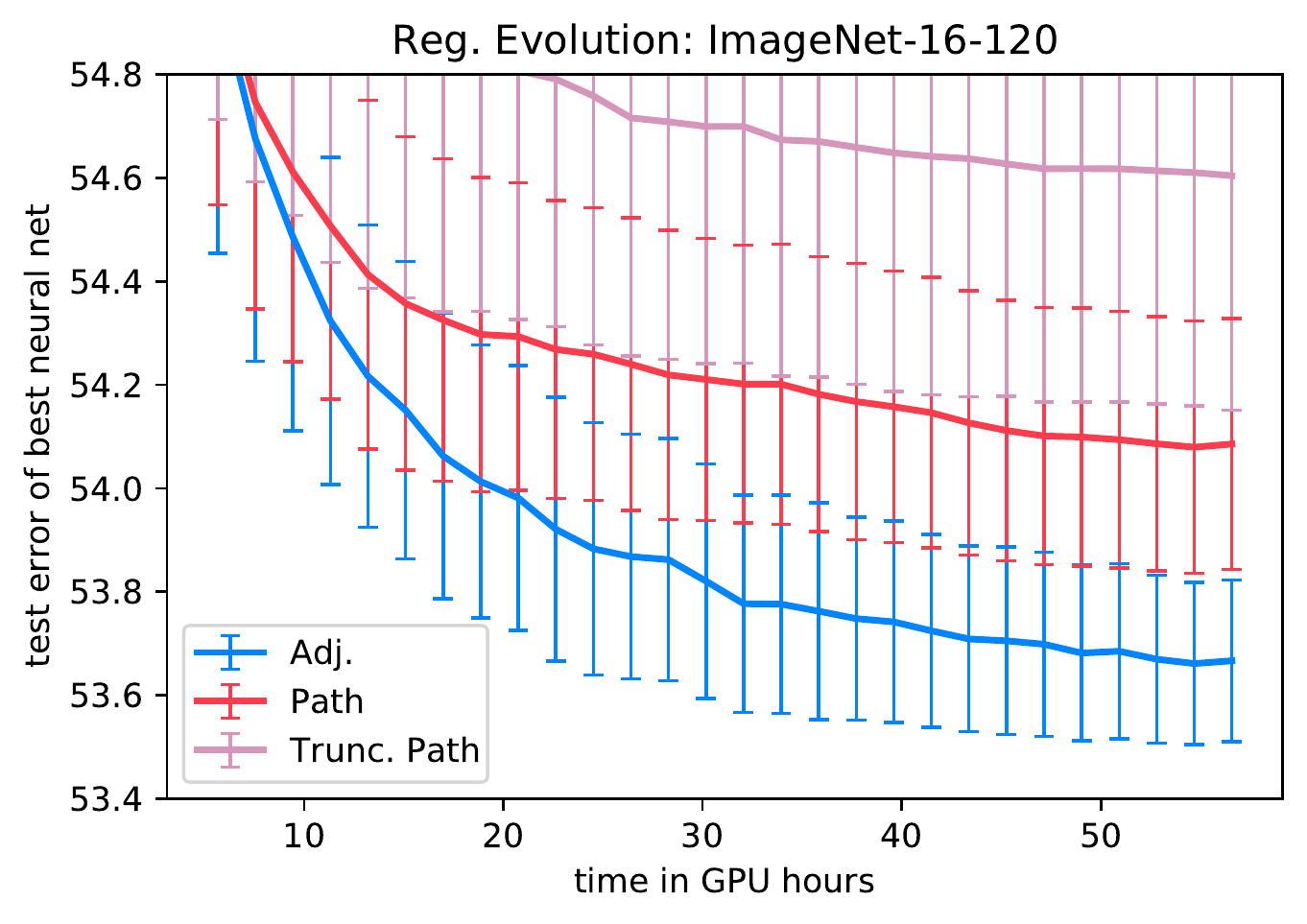}
\includegraphics[width=0.33\textwidth]{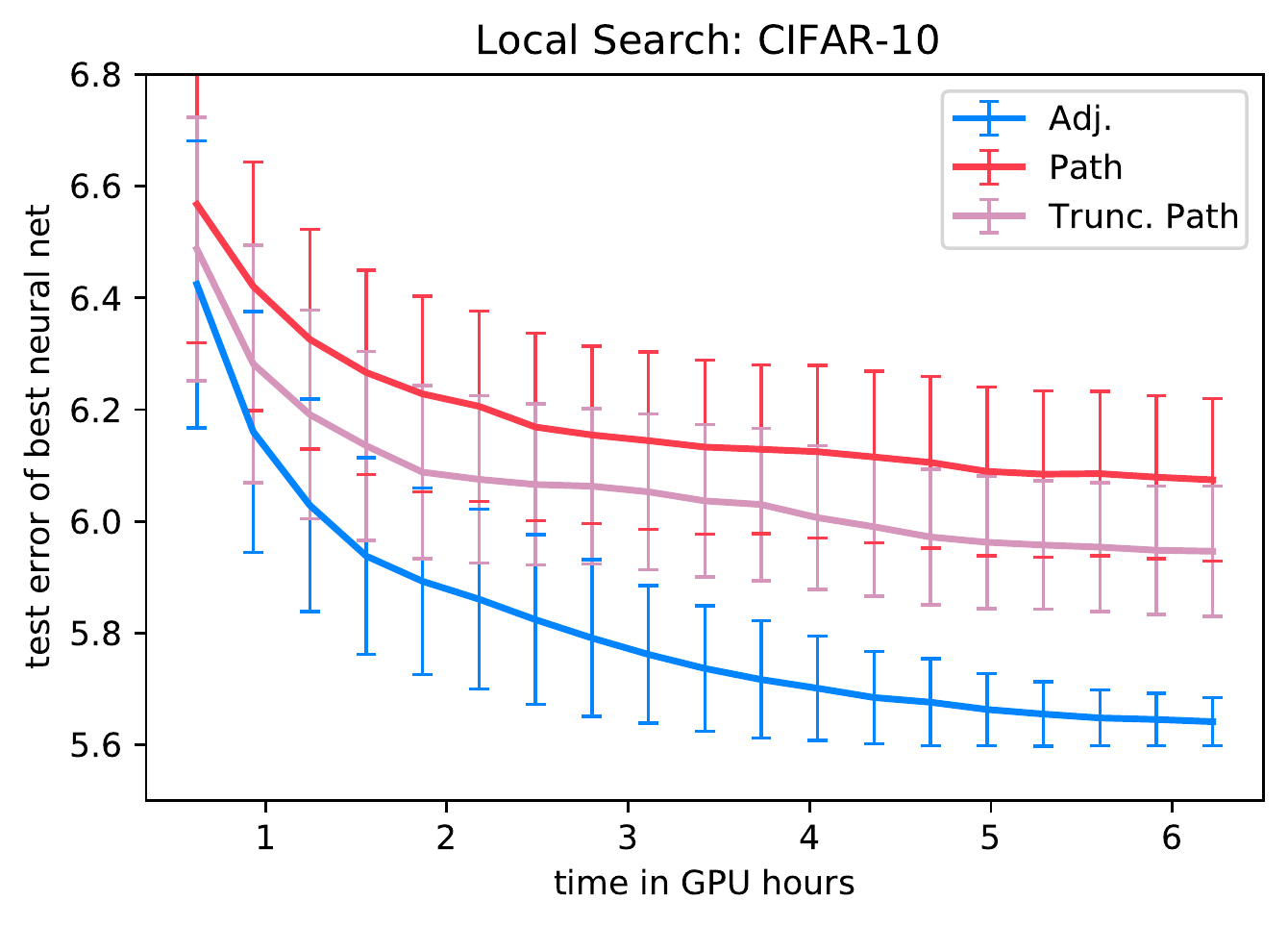}
\hspace{-3pt}
\includegraphics[width=0.33\textwidth]{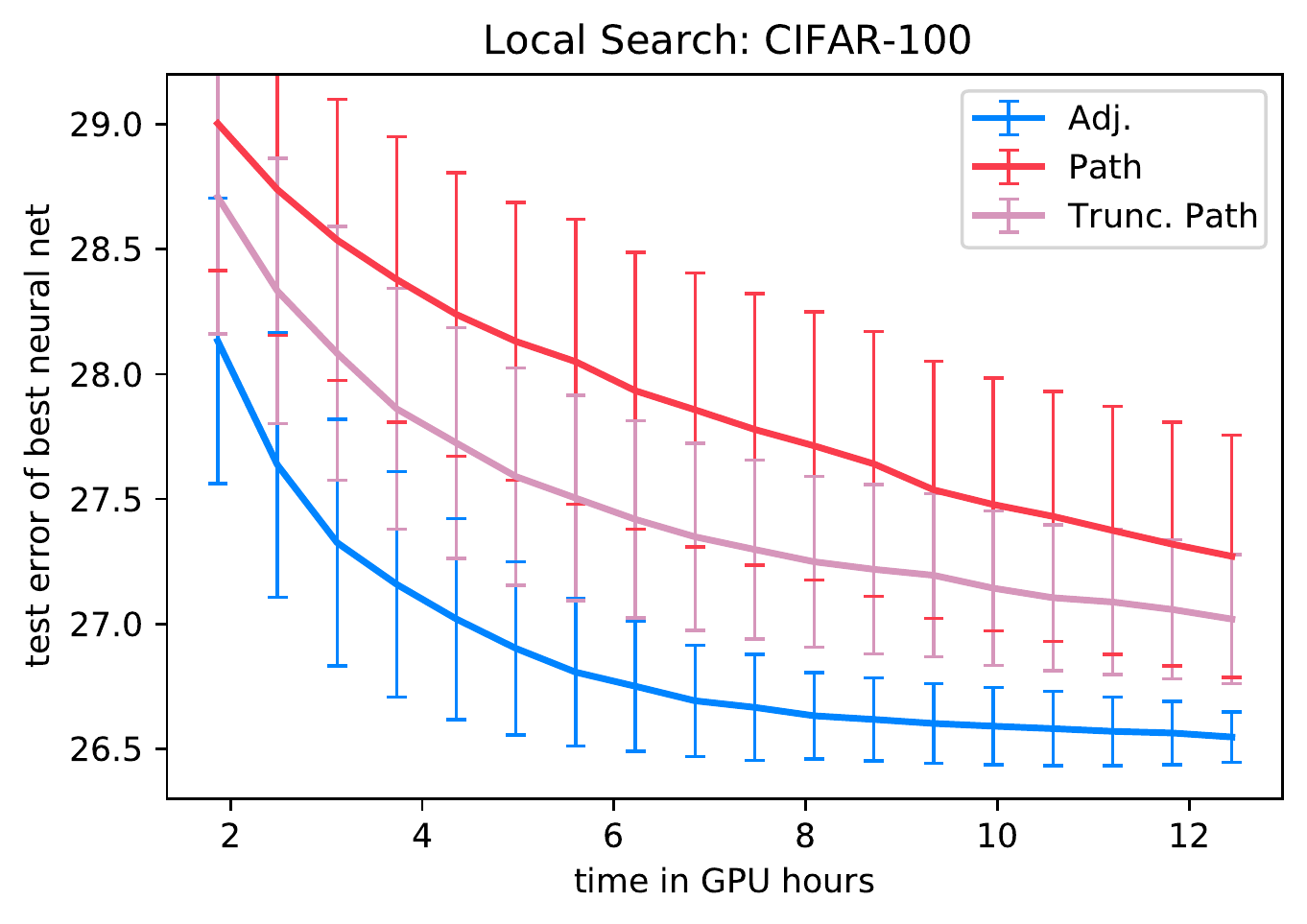}
\hspace{-3pt}
\includegraphics[width=0.32\textwidth]{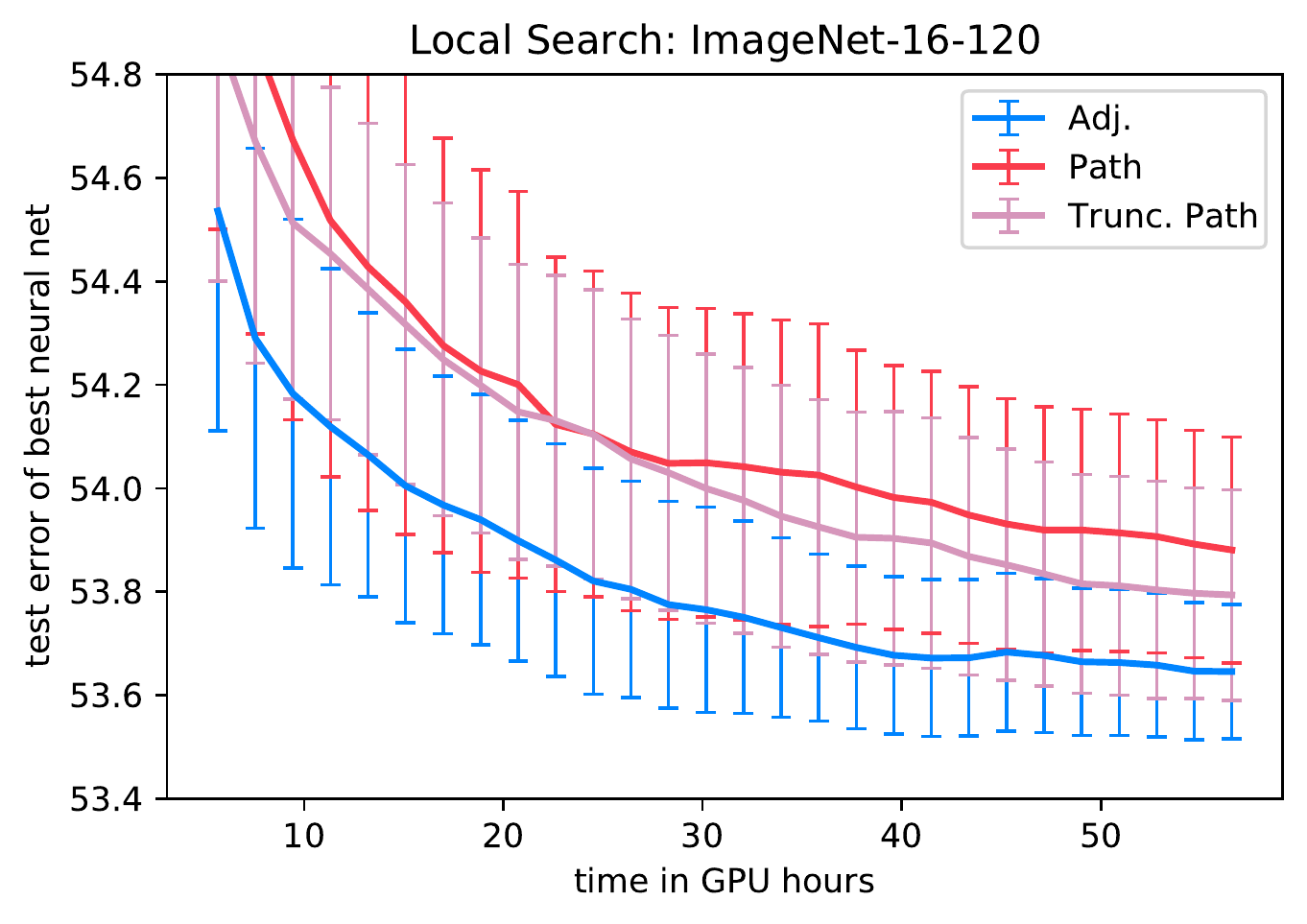}
\includegraphics[width=0.33\textwidth]{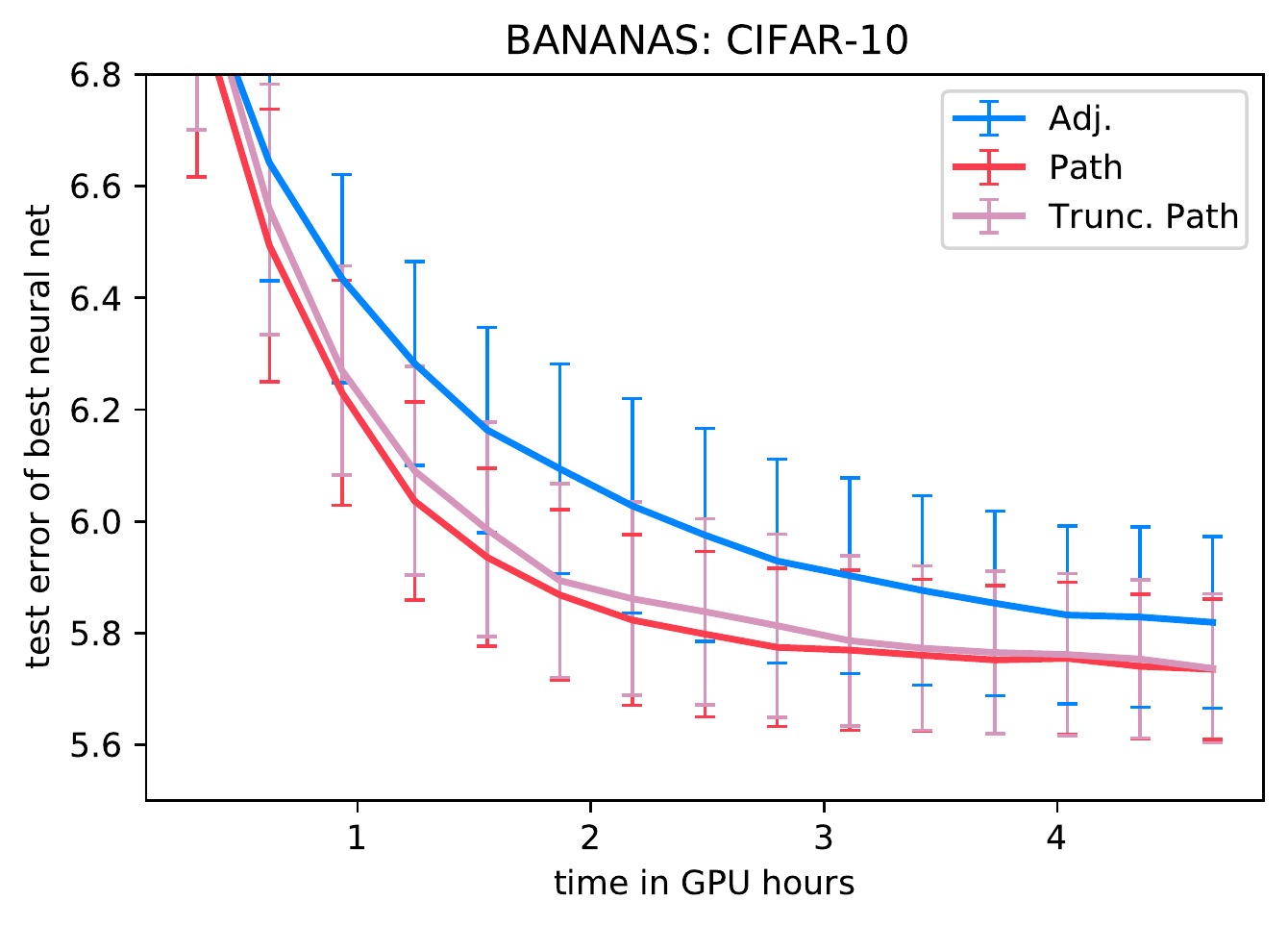}
\hspace{-3pt}
\includegraphics[width=0.33\textwidth]{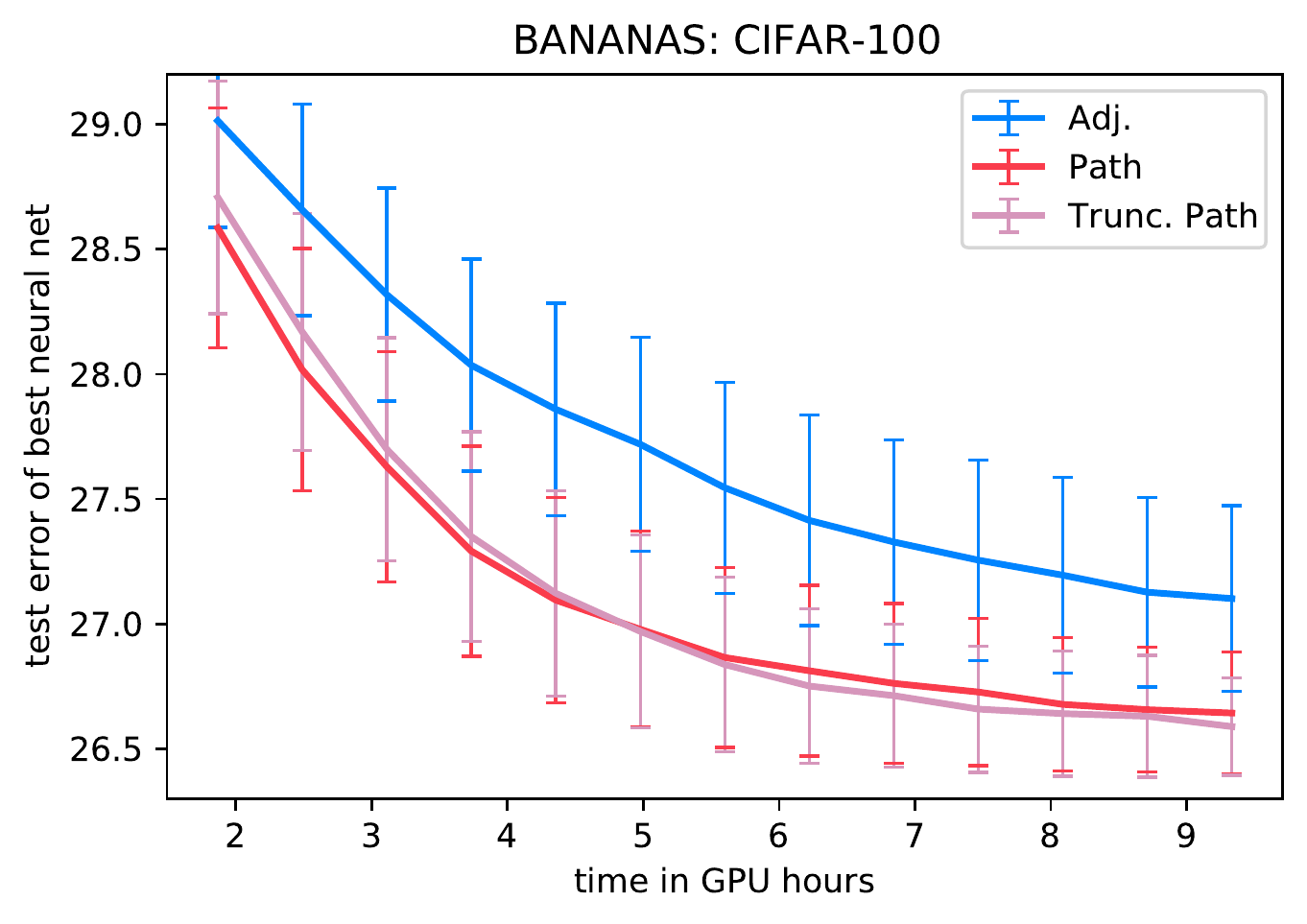}
\hspace{-3pt}
\includegraphics[width=0.32\textwidth]{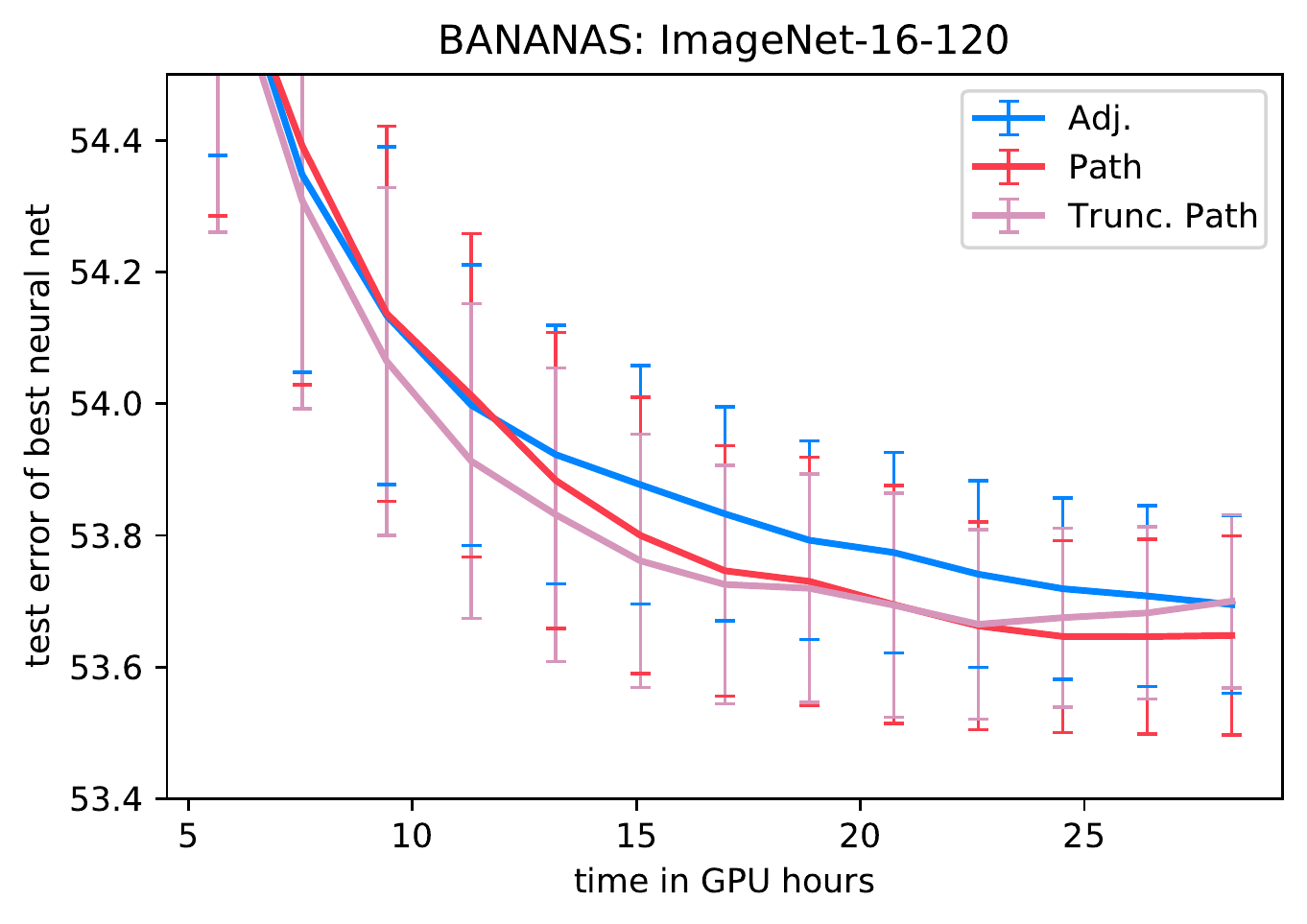}
\includegraphics[width=0.33\textwidth]{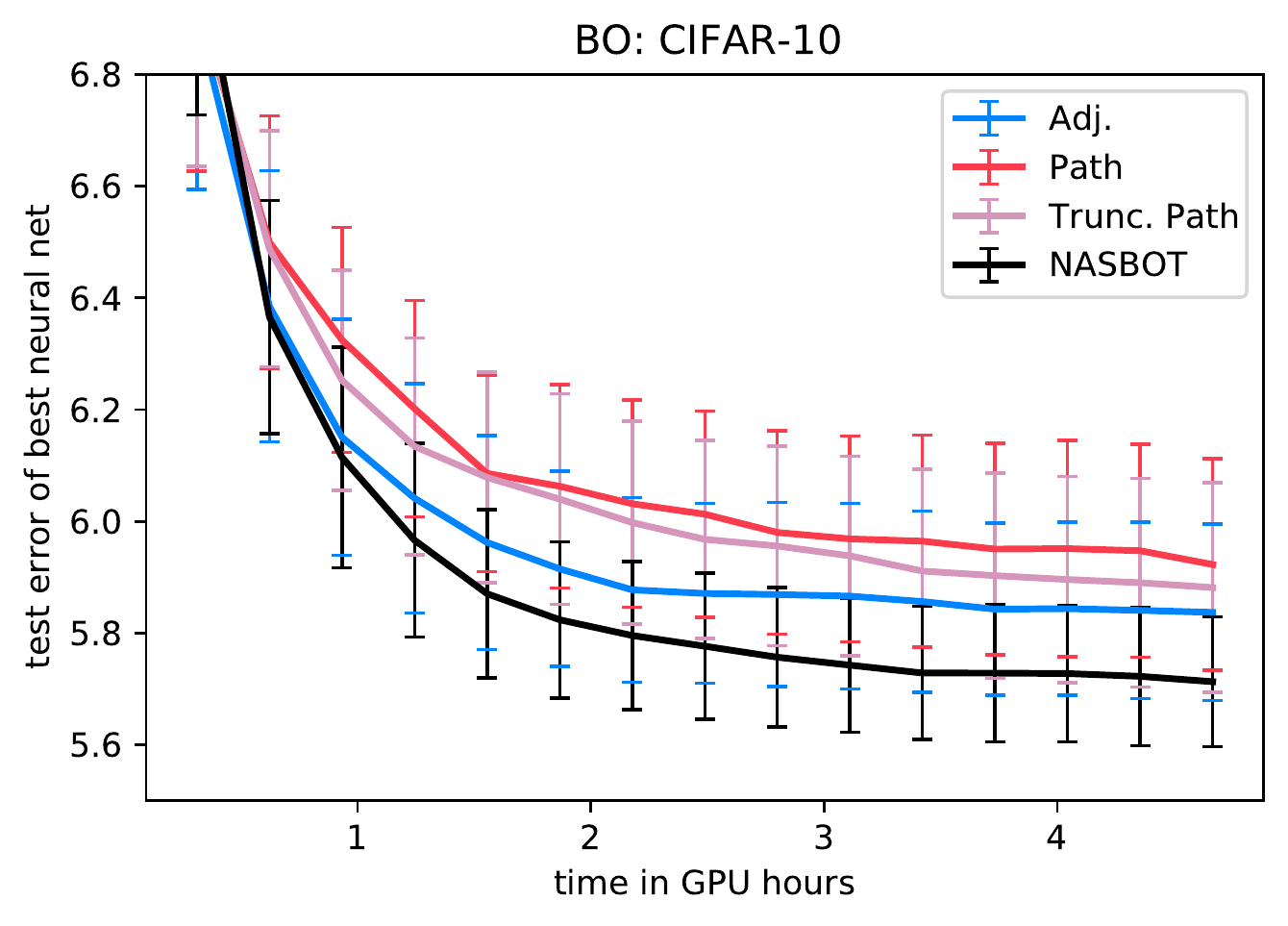}
\hspace{-3pt}
\includegraphics[width=0.33\textwidth]{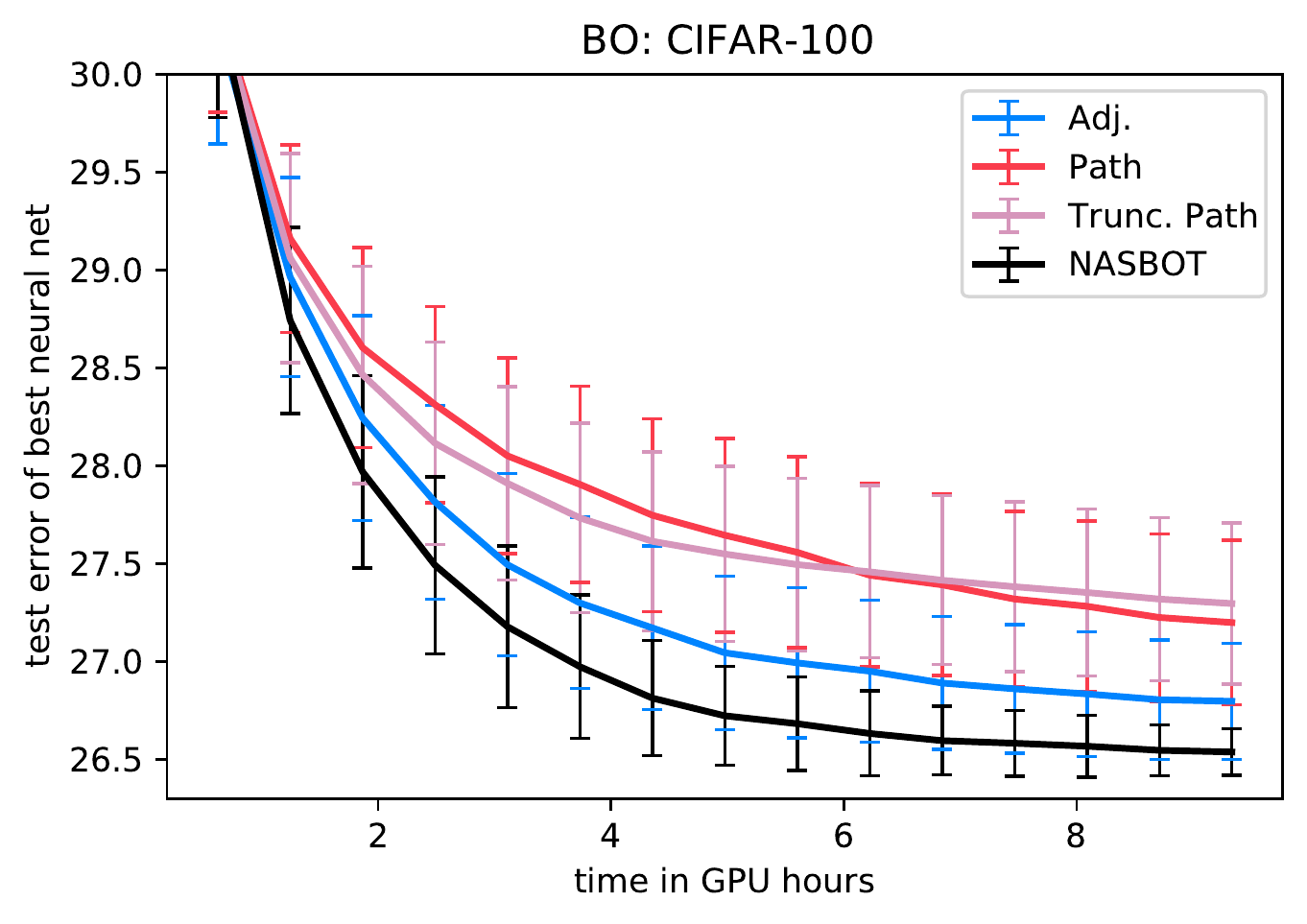}
\hspace{-3pt}
\includegraphics[width=0.32\textwidth]{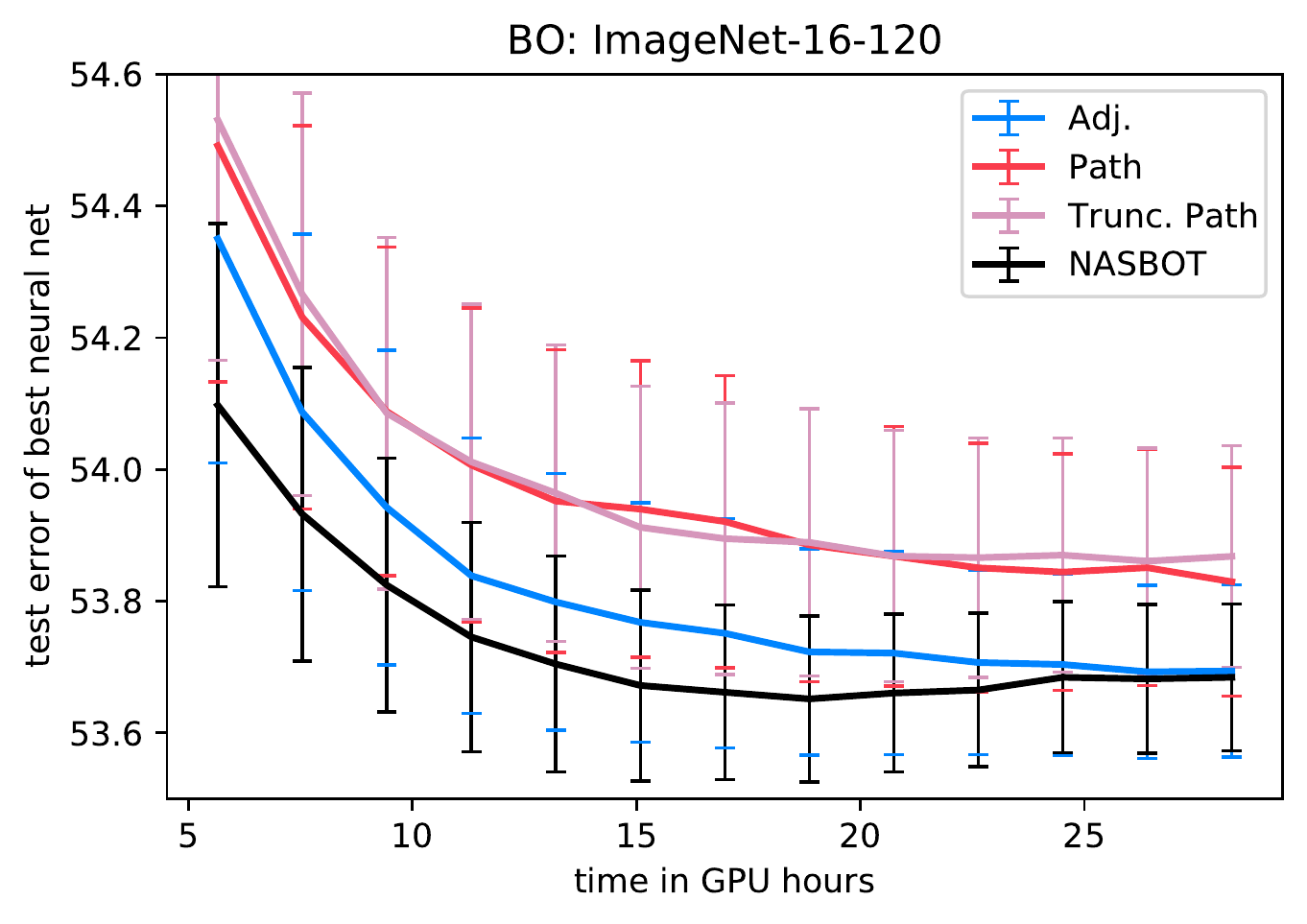}
\caption{
Experiments on NASBench-201 with different encodings, keeping all but one subroutine
fixed: \emph{perturb architecture} (Reg. evolution (top row), local search (second row)),
\emph{train predictor model} (BANANAS (third row), Bayesian optimization (bottom row)).
}
\label{fig:201_results}
\end{figure*}

\subsection{Best practices for NAS}
Many authors have called for improving the reproducibility and fairness in experimental comparisons in
NAS research~\citep{randomnas, nasbench, yang2019evaluation}, which has led to the release of a
NAS best practices checklist~\citep{lindauer2019best}.
We address each section and we encourage future work to do the same.
\begin{itemize}
    \item \textbf{Best practices for releasing code.}
    We released our code publicly. 
    We used the NASBench-101 and NASBench-201 datasets, 
    so questions about training pipeline, evaluation,
    and hyperparameters for the final evaluation do not apply.
    \item \textbf{Best practices for comparing NAS methods.}
    We made fair comparisons due to our use of NASBench-101 and NASBench-201.
    We did run ablation studies and ran random search.
    We performed 300 trials of each experiment on NASBench-101 and NASBench-201.
    \item \textbf{Best practices for reporting important details.}
    We used the hyperparamters straight from the open source repositories, 
    with a few small exceptions listed earlier in this section.
\end{itemize}

\end{document}